\documentclass{article}

\PassOptionsToPackage{square,sort&compress,numbers}{natbib}
\usepackage[final]{neurips_2022}

\usepackage[utf8]{inputenc} % allow utf-8 input
\usepackage[T1]{fontenc}    % use 8-bit T1 fonts
\usepackage[hidelinks]{hyperref}       % hyperlinks
\usepackage{url}            % simple URL typesetting
\usepackage{booktabs}       % professional-quality tables
\usepackage{amsfonts}       % blackboard math symbols
\usepackage{nicefrac}       % compact symbols for 1/2, etc.
\usepackage{microtype}      % microtypography
\usepackage{xcolor}         % colors
\usepackage{wrapfig}
\usepackage{verbatim}
\usepackage{amsmath}
\usepackage{lipsum}
\usepackage[toc,page,header]{appendix}
\usepackage{minitoc}
\usepackage{multibib}
\newcites{supp}{Supplementary References}

\usepackage{pifont}         % for cmark and xmark symbols
\newcommand{\cmark}{\ding{51}}%
\newcommand{\xmark}{\ding{55}}%

\doparttoc % Tell to minitoc to generate a toc for the parts
\faketableofcontents % Run a fake tableofcontents command for the partocs

\usepackage{amsthm}
\makeatletter
\newtheorem*{rep@theorem}{\rep@title}
\newcommand{\newreptheorem}[2]{%
\newenvironment{rep#1}[1]{%
 \def\rep@title{#2 \ref{##1}}%
 \begin{rep@theorem}}%
 {\end{rep@theorem}}}
\makeatother

\DeclareMathOperator*{\argmin}{arg\,min}
\newcommand{\norm}[1]{\left\lVert#1\right\rVert}
\newcommand{\snorm}[1]{\lVert#1\rVert}
\newcommand{\pder}[2]{\frac{\partial#2}{\partial #1}}
\newcommand{\spder}[2]{\partial_#1#2}
\newcommand{\der}[2]{\frac{\mathrm{d}#2}{\mathrm{d} #1}}
\newcommand{\sder}[2]{\mathrm{d}_#1#2}
\newcommand{\Id}{\mathrm{Id}}
\newcommand{\evalat}[2]{\left. #1 \right\rvert_{#2}}

\newtheorem{theorem}{Theorem}
\newreptheorem{theorem}{Theorem}

\newtheorem{assumption}{Assumption}
\newreptheorem{assumption}{Assumption}
\newtheorem{lemma}{Lemma}
\newtheorem{remark}{Remark}
\newreptheorem{lemma}{Lemma}
\newtheorem{corollary}{Corollary}
\newreptheorem{corollary}{Corollary}

\newcommand{\Lt}{\mathcal{L}}
\newcommand{\Li}{L^\mathrm{learn}}
\newcommand{\Lo}{L^\mathrm{eval}}
\newcommand{\phisb}{\phi_\beta^*}
\newcommand{\phisbt}{\phi_{\theta, \beta}^*}
\newcommand{\phisz}{\phi_0^*}
\newcommand{\phiszt}{\phi_{\theta, 0}^*}
\newcommand{\phist}{\phi_{\theta}^*}
\newcommand{\phisi}{\phi_\infty^*}
\newcommand{\phihb}{\hat{\phi}_\beta}
\newcommand{\phihz}{\hat{\phi}_0}

\usepackage[inline]{enumitem}

\usepackage{graphicx}

\usepackage{lscape}  % add landscape pages for large tables

\usepackage{array} % allow for fixed width columns
\newcolumntype{L}[1]{>{\raggedright\let\newline\\\arraybackslash\hspace{0pt}}m{#1}}

\title{A contrastive rule for meta-learning}

\author{%
  Nicolas Zucchet\thanks{Equal contribution; arbitrary ordering.}\\
  Department of Computer Science\\
  ETH Z\"{u}rich\\
  \texttt{nzucchet@inf.ethz.ch}\\
  \And
  Simon Schug$^*$\\
  Institute of Neuroinformatics\\
  University of Z\"{u}rich \& ETH Z\"{u}rich\\
  \texttt{sschug@ethz.ch}\\
  \AND
  Johannes von Oswald$^*$\\
  Department of Computer Science\\
  ETH Z\"{u}rich\\
  \texttt{voswaldj@ethz.ch} \\
  \And
  Dominic Zhao\\
  Institute of Neuroinformatics\\
  University of Z\"{u}rich \& ETH Z\"{u}rich\\
  \texttt{dozhao@ethz.ch} \\
  \AND
  João Sacramento\\
  Institute of Neuroinformatics\\
  University of Z\"{u}rich \& ETH Z\"{u}rich\\
  \texttt{rjoao@ethz.ch}\\
}

\let\oldmaketitle\maketitle
\renewcommand{\maketitle}{\oldmaketitle\setcounter{footnote}{0}}

\begin{document}

\maketitle

\begin{abstract}
Humans and other animals are capable of improving their learning performance as they solve related tasks from a given problem domain, to the point of being able to learn from extremely limited data. While synaptic plasticity is generically thought to underlie learning in the brain, the precise neural and synaptic mechanisms by which learning processes improve through experience are not well understood. Here, we present a general-purpose, biologically-plausible meta-learning rule which estimates gradients with respect to the parameters of an underlying learning algorithm by simply running it twice. Our rule may be understood as a generalization of contrastive Hebbian learning to meta-learning and notably, it neither requires computing second derivatives nor going backwards in time, two characteristic features of previous gradient-based methods that are hard to conceive in physical neural circuits. We demonstrate the generality of our rule by applying it to two distinct models: a complex synapse with internal states which consolidate task-shared information, and a dual-system architecture in which a primary network is rapidly modulated by another one to learn the specifics of each task. For both models, our meta-learning rule matches or outperforms reference algorithms on a wide range of benchmark problems, while only using information presumed to be locally available at neurons and synapses. We corroborate these findings with a theoretical analysis of the gradient estimation error incurred by our rule.\footnote{Code available at \url{https://github.com/smonsays/contrastive-meta-learning}}
\end{abstract}

%%%%%%%%%%%%%%%%%%%%%%%%%%%%%%%%%%%%%%%%%%%%%%%%%%%%%%%%%%%%
\section{Introduction}

The seminal study of \citet{harlow_formation_1949} established that humans and non-human primates can become better at learning when presented with a series of learning tasks which share a certain common structure. To achieve this, the brain must extract and encode whichever aspects are common within a problem domain, in such a way that future learning performance is improved. This capacity, which we refer to as meta-learning, confers great evolutionary advantage to an organism over another that must face new tasks starting from \emph{tabula rasa}. The neural and synaptic basis of this higher-order form of learning is largely unknown and theories are notably scarce \citep{brea_does_2016}. The present work focuses on developing one such theory.

Formally, we define learning as the optimization of a data-dependent objective function with respect to learnable parameters, following the prevalent view in machine learning \citep{richards_deep_2019}. Meta-learning can be straightforwardly accommodated for in this framework by first specifying a learning algorithm through a set of meta-parameters, and then measuring post-learning performance through a meta-objective function \citep{schmidhuber_evolutionary_1987,bengio_learning_1990,chalmers_evolution_1991,thrun_learning_1998,hochreiter_learning_2001}. Formulated as such, meta-learning corresponds to a hierarchical optimization problem, where lower-level parameters are optimized to learn the specifics of each task, and meta-parameters are adapted over tasks to improve overall learning performance.

An essential question in this framework is how to optimize meta-parameters. In current deep learning practice, meta-parameters are almost always learned by backpropagation-through-learning, an instance of backpropagation-through-time \citep{werbos_backpropagation_1990}. While a number of biologically-plausible designs \citep{whittington_theories_2019,richards_dendritic_2019,roelfsema_control_2018,richards_deep_2019,lillicrap_backpropagation_2020} have been developed for the standard error backpropagation algorithm for feedforward neural networks  \citep{werbos_beyond_1974,rumelhart_learning_1986}, backpropagation-through-learning suffers from a number of issues which appear to be fundamentally difficult to overcome in biological circuits. For example, when learning involves optimizing synaptic connection weights -- as it is presumed to be the case in the brain -- implementing backpropagation-through-learning would entail backtracking through a sequence of synaptic changes in reverse-time order, while carrying out operations which would require knowledge of all synaptic weights to be available at a single synapse. This is clearly at odds with what is currently known about synaptic plasticity. Thus, calculating meta-parameter gradients by backpropagation is both computationally expensive, and hard to reconcile with biological constraints.

Here we present a meta-learning rule for adapting meta-parameters which does not exhibit such issues. Instead of backpropagating through a learning process, our rule estimates meta-parameter gradients by running the underlying learning algorithm twice: learning a task is followed by a second run to solve an augmented learning problem which includes the meta-objective. Our rule has a number of appealing properties:
\begin{enumerate*}[label=(\arabic*)]
\item it runs forward in time, making the learning rule causal;
\item implementing it only requires temporarily buffering one intermediate state;
\item it does not evaluate second derivatives, thus avoiding accessing information that is non-local to a parameter; and
\item it approximates meta-gradients as accurately as needed.
\end{enumerate*}
Furthermore, our rule is generically applicable and it can be used to learn any meta-parameter which influences the meta-objective function.

The local and causal nature of our rule allows us to develop a theory of meta-plastic synapses, which slowly consolidate information over tasks in their internal hidden states or in their synaptic weights. We show through experiments that, when governed by our meta-learning rule, such slow adaptation processes result in improved learning performance in a variety of benchmark problems and network architectures, from deep convolutional to recurrent spiking neural networks, on both supervised and reinforcement learning paradigms. Moreover, we find that our meta-learning rule performs as well or better than reference methods, including backpropagation-through-learning, and we provide a theoretical bound for its meta-gradient estimation error which is confirmed by our experimental findings. Thus, our results demonstrate that gradient-based meta-learning is possible with local learning rules, and suggest ways by which slower synaptic processes in the brain optimize the performance of faster learning processes.

%%%%%%%%%%%%%%%%%%%%%%%%%%%%%%%%%%%%%%%%%%%%%%%%%%%%%%%%%%%%
\section{Background and problem setup}
The goal of meta-learning is to improve the performance of a learning algorithm through experience. We begin by formalizing this goal as a mathematical optimization problem and outlining its solution with standard gradient-based methods. The approach we present below underlies a large body of work studying meta-learning in neural networks \citep[e.g.,][]{sutton_adapting_1992,thrun_learning_1998,andrychowicz_learning_2016,finn_model-agnostic_2017}. We also discuss why these standard methods may be deemed unsatisfactory as models of meta-learning in the brain.

\paragraph{Problem setup.} Formally, we wish to optimize the meta-parameters $\theta$ of an algorithm which learns to solve a given task $\tau$ by changing the parameters $\phi$ of a model. Each task is drawn from a distribution $p(\tau)$ representing the problem domain and comes with an associated loss function $L_\tau^\mathrm{learn}(\phi, \theta)$, which depends on some data $D_\tau^\mathrm{learn}$. The goal of learning is to minimize this loss while keeping the meta-parameters $\theta$ fixed; we denote the outcome of learning task $\tau$ by $\phi_{\theta, \tau}^*$. The subscript $\theta$ in $\phi_{\theta,\tau}^*$ is here to emphasize that the solution of a task implicitly depends on the meta-parameters $\theta$ used during learning. Learning performance is then evaluated by measuring again a loss function $L^\mathrm{eval}_\tau(\phi_{\theta,\tau}^*, \theta)$, defined on new evaluation data $D^\mathrm{eval}_\tau$ from the same task. The meta-objective is this evaluation loss, averaged over tasks. Hence, we formalize meta-learning as a bilevel optimization problem, which can be compactly written as follows:
\begin{equation}
\begin{aligned}
  \label{eq:bilevel-meta-learning-problem}
  \min_\theta \; \mathbb{E}_{\tau \sim p(\tau)} \! \left[ L_\tau^\mathrm{eval}(\phi_{\theta,\tau}^*, \theta)\right] \quad
  \;\,\mathrm{s.t.} \; \; \phi_{\theta,\tau}^* \in \mathop{\mathrm{arg\,min}}_\phi L_\tau^\mathrm{learn}(\phi, \theta).
\end{aligned}
\end{equation}

In this paper, we approach problem \eqref{eq:bilevel-meta-learning-problem} with stochastic gradient descent, which uses meta-gradient information to update meta-parameters after learning a task (or a minibatch of tasks) presented by the environment. For a given task $\tau$ we thus need to compute the meta-gradient
\begin{equation}
  \label{eq:meta-gradient-for-task-tau}
  \nabla_{\theta,\tau} := \left(\frac{\mathrm{d}}{\mathrm{d}\theta} L_\tau^\mathrm{eval}(\phi_{\theta,\tau}^*, \theta) \right)^\top.
\end{equation}
The implicit dependence of $\phi_{\theta,\tau}^*$ on the meta-parameters $\theta$ complicates the computation of the meta-gradient; differentiating through the learning algorithm efficiently is a central question in gradient-based meta-learning. We next review two major known ways of doing so.

\paragraph{Review of backpropagation-through-learning.} A common strategy followed in previous work \citep[cf.][]{hospedales_meta-learning_2020} is to replace the solution $\phi_{\theta,\tau}^*$ to a learning task by the result $\phi_{\theta,\tau,T}$ obtained after applying a differentiable learning algorithm for $T$ time steps, not necessarily until convergence. One advantage of this formulation is that the computational graph for $\phi_{\theta,\tau,T}$ is explicitly available. Thus, backpropagation can be invoked to compute the meta-gradient $\nabla_{\theta,\tau}$, yielding what we refer to as backpropagation-through-learning. This approach is hardly biologically-plausible, as it requires storing and revisiting the parameter trajectory $\{\phi_t\}_{t=1}^T$ backwards in time, from $t=T$ to $t=0$. Moreover, when the learning algorithm which produces $\phi_{\theta,\tau,T}$ is itself gradient-based, as it typically is in deep learning, differentiating through learning gives rise to second derivatives. These second-order terms involve cross-parameter dependencies that are difficult to resolve with local processes.

\paragraph{Review of implicit differentiation.} An alternative line of methods \citep{bengio_gradient-based_2000,pedregosa_hyperparameter_2016,rajeswaran_meta-learning_2019,lorraine_optimizing_2020} approaches problem \eqref{eq:bilevel-meta-learning-problem} through the implicit function theorem \citep{dontchev_implicit_2009}. This theorem provides conditions under which the meta-gradient $\nabla_{\theta,\tau}$ is well-defined, while also providing a formula for it. Over backpropagation-through-learning, this approach has the advantages that it does not require storing parameter trajectories $\{\phi_t\}_{t=1}^T$, and that it is agnostic to which algorithm is used to learn a task. However, the meta-gradient formula provided by the implicit function theorem is difficult to evaluate directly for neural network models, as it includes the inverse learning loss Hessian. This makes it hard to design biologically-plausible meta-learning algorithms based directly on the implicit meta-gradient expression. We refer to Section~\ref{sec_app:review_implicit} for more details and an expanded discussion on this class of meta-learning methods.

%%%%%%%%%%%%%%%%%%%%%%%%%%%%%%%%%%%%%%%%%%%%%%%%%%%%%%%%%%%%
\section{Contrastive meta-learning}
Here we present a new meta-learning rule which is generically applicable to meta-learning problems of the form \eqref{eq:bilevel-meta-learning-problem}. Our rule is gradient-following, and therefore scalable to neural network problems involving high-dimensional meta-parameters, while being simpler to conceive in biological neural circuits than the standard gradient-based methods reviewed in the previous section.

To derive our meta-learning rule we first introduce an auxiliary objective function which mixes the two levels of the bilevel optimization problem \eqref{eq:bilevel-meta-learning-problem}:
\begin{equation}
  \label{eq:augmented-loss}
  \mathcal{L}_\tau(\phi, \theta, \beta) = L_\tau^\mathrm{learn}(\phi,\theta) + \beta L_\tau^\mathrm{eval}(\phi,\theta).
\end{equation}
We refer to $\mathcal{L}_\tau(\phi,\theta,\beta)$ as the augmented loss function. This auxiliary loss depends on a new scalar parameter $\beta \in \mathbb{R}$, which we call the nudging strength. Positive values of $\beta$ nudge learning towards the meta-objective associated with task $\tau$. Thus, we can define a family of auxiliary learning problems through the augmented loss $\mathcal{L}_\tau$ by varying the nudging strength $\beta$ away from zero. We denote the solutions to these auxiliary learning problems by 
\begin{equation}
  \phi^*_{\theta,\beta,\tau} \in \mathop{\mathrm{arg\,min}}_\phi \mathcal{L}_\tau(\phi, \theta, \beta),
\end{equation}
and we use $\hat{\phi}_{\theta,\beta,\tau}$ to distinguish approximate model parameters found in practice with some learning algorithm from the true minimizers $\phi^*_{\theta,\beta,\tau}$. Note that for the special case of $\beta=0$, we recover a solution $\phi^*_{\theta,0,\tau}$ of the original learning task defined by $L_\tau^\mathrm{learn}(\phi,\theta)$.

Our contrastive meta-learning rule prescribes the following change to the meta-parameters $\theta$ after encountering learning task $\tau$:
\begin{equation}
  \label{eq:CML-delta-theta}
  \Delta_{\theta,\tau} := - \frac{1}{\beta}\left(\frac{\partial \mathcal{L}_\tau}{\partial \theta} (\hat{\phi}_{\theta,\beta,\tau}, \theta, \beta) - \frac{\partial\mathcal{L}_\tau}{\partial \theta}(\hat{\phi}_{\theta,0,\tau}, \theta, 0)\right)^\top.
\end{equation}
This rule contrasts information over two model parameter settings, $\hat{\phi}_{\theta,0,\tau}$ and $\hat{\phi}_{\theta,\beta,\tau}$; it may be understood as a generalization to meta-learning of a classical recurrent neural network learning algorithm known as contrastive Hebbian learning \citep{peterson_mean_1987,movellan_contrastive_1991,baldi_contrastive_1991,oreilly_biologically_1996,scellier_equilibrium_2017}. Intuitively, as we compute the solution to the augmented learning problem with $\beta>0$, we nudge our learning algorithm towards a parameter setting $\hat{\phi}_{\theta,\beta,\tau}$ that would have been better in terms of the meta-objective --- that we wish our algorithm had actually reached, without needing the meta-objective to influence the learning process.

Our rule implements meta-learning by gradient descent when the learning solutions $\hat{\phi}_{\theta,0,\tau}$ and $\hat{\phi}_{\theta,\beta,\tau}$ are exact and as $\beta \to 0$. This important property can be shown by invoking the equilibrium propagation theorem \citep{scellier_equilibrium_2017,scellier_deep_2021} discovered and proved by Scellier and Bengio; we restate this result and present the technical conditions for applying it to meta-learning in Section~\ref{sec_app:equilibrium_propagation}. Critically, $\Delta_{\theta,\tau}$ estimates the meta-gradient $\nabla_{\theta,\tau}$ using only partial derivative information and without ever directly calculating the total derivative in \eqref{eq:meta-gradient-for-task-tau}. Depending on the model, partial derivatives of the augmented loss $\mathcal{L}_\tau$ may be easy to calculate analytically and implement, or they may require dedicated neural circuits for their evaluation; we return to this point in the next section.

We recall that the two points $\hat{\phi}_{\theta,0,\tau}$ and $\hat{\phi}_{\theta,\beta,\tau}$ which appear in \eqref{eq:CML-delta-theta} respectively correspond to approximate solutions of the original and the augmented learning problems. Thus, the information required to implement our rule can be collected causally by invoking the learning algorithm for a second time, after the actual task has been learned, while buffering information across the two runs. In contrast to backpropagation-through-learning, this process runs forward in time, it only requires keeping a single intermediate state in short-term memory, and it is entirely agnostic to which underlying learning algorithm is used. Moreover, as we will show in the theoretical results, its precision can be varied; the same rule can produce both coarse- and fine-grained meta-gradient estimates as needed, by varying the amount of resources spent in learning and by controlling the nudging strength $\beta$.

%%%%%%%%%%%%%%%%%%%%%%%%%%%%%%%%%%%%%%%%%%%%%%%%%%%%%%%%%%%%
\section{Models}
In the previous section, our contrastive meta-learning rule was presented in its general form. We now describe two concrete neural models that provide complementary views on how meta-learning could be conceived in the brain. We study the specific meta-learning rules arising from the application of the update \eqref{eq:CML-delta-theta} to each case and discuss their implementation with biological neural circuitry.

\subsection{Synaptic consolidation as meta-learning}
\label{sect:synaptic-model}
We first use our general contrastive meta-learning rule \eqref{eq:CML-delta-theta} to derive meta-plasticity rules for a complex synapse model which has been featured in prior meta-learning \citep{rajeswaran_meta-learning_2019,chen_modular_2020} and continual learning \citep{zenke_continual_2017,kirkpatrick_overcoming_2017} work. Biological synapses are complex devices which comprise components that adapt at multiple time scales. Beyond changes induced by standard long-term potentiation and depression protocols lasting minutes to several hours, synapses exhibit activity-dependent plasticity at much longer time scales \citep{abraham_metaplasticity_2008,fusi_cascade_2005,ziegler_synaptic_2015}. While previous work has focused on characterizing memory retention in more realistic synapse models, here we study how such slow synaptic consolidation processes may support fast future learning through our contrastive meta-learning rule.

In the model we consider, besides a synaptic weight $\phi$ which influences postsynaptic activity, each synapse has an internal consolidated state $\omega$ towards which the weight is attracted whenever the synapse changes. We further allow the attraction strength $\lambda$ to vary over synapses; its reciprocal $\lambda^{-1}$ plays a role similar to a learning rate. For this model the meta-parameters are therefore $\theta=\{\lambda, \omega\}$. We model the interaction between these three components through a quadratic function, which is added to the task-specific learning loss $l^\mathrm{learn}_\tau(\phi)$:
\begin{equation}
\label{eq:synaptic-model-learning-loss}
  L^\mathrm{learn}_\tau(\phi,\theta) = l^\mathrm{learn}_\tau(\phi) + \frac{1}{2}\sum_{i=1}^{|\phi|}\lambda_i(\omega_i - \phi_i)^2.
\end{equation}
In machine learning terms, we regularize the learning loss with a quadratic regularizer. On the other hand, the evaluation loss function  $L_\tau^\mathrm{eval}(\phi)$ depends only on the synaptic weights $\phi$ such that the meta-parameters $\theta$ only influence learning, not prediction.

The partial derivatives which appear in our contrastive meta-learning rule \eqref{eq:CML-delta-theta} can be analytically obtained for this synaptic model. A calculation yields the meta-plasticity rules
\begin{equation}
\label{eq:synaptic-model-updates}
  \Delta_{\omega,\tau} = \frac{\lambda}{\beta} \left(\hat{\phi}_{\theta,\beta,\tau} - \hat{\phi}_{\theta,0,\tau}\right) \quad\text{and}\quad \Delta_{\lambda,\tau} = \frac{1}{2\beta}\left[(\hat{\phi}_{\theta,0,\tau} - \omega)^2 - (\hat{\phi}_{\theta,\beta,\tau} - \omega)^2 \right],
\end{equation}
where all operations are carried out elementwise. Contrastive meta-learning thus offers a principled way to slowly (over learning tasks) consolidate information in the internal states of complex synapses to improve future learning performance. Critically, it leads to meta-plasticity rules that are entirely local to a synapse and are independent of the method used to learn. Our meta-plasticity rules can thus be flexibly applied to improve the performance of any learning algorithm, including a host of biologically-plausible learning rules, from precise neuron-specific error backpropagation circuits \citep[][]{whittington_approximation_2017,payeur_burst-dependent_2021} to stochastic perturbation reinforcement rules \citep{xie_learning_2004}. The only requirement our theory makes is that learning corresponds to the optimization of an objective.

\subsection{Learning by top-down modulation}
\label{sect:modulation}
The second model that we consider is inspired by the modulatory role that is attributed to top-down inputs from higher- to lower-order brain areas. Such modulatory inputs often feature in neural theories of attention and contextual processing \citep{miller_integrative_2001,rikhye_toward_2018,titley_toward_2017}. Here, we explore the possibility that they subserve fast learning of new tasks. We incorporate this insight into a simple meta-learning model, where learning a task $\tau$ corresponds to finding the right pattern of task-specific modulation $\phi_{\theta,\tau}^*$, and meta-learning corresponds to changing synaptic weights $\theta$. Unlike in the complex synapse model presented in the previous section, here we interpret the task-specific parameters $\phi_\tau$ as patterns of neural activity, not synaptic weights. This implies that, if meta-learning succeeds, it becomes possible to learn new tasks on the fast neural time scale without evoking synaptic plasticity.
%that they provide a rich basis for solving tasks by modulation alone.
% In the cortex, sensory and motor neurons characteristically receive substantial input from higher-order areas such as the prefrontal cortex and higher-order thalamic nuclei, which is thought to convey context- and task-specific information.

More concretely, we take as modulatory inputs a multiplicative gain $g$ and an adaptive threshold $b$ per neuron, as done in previous work \citep{zintgraf_fast_2019,perez_film_2018}. Rapid (input-dependent) multiplicative and additive modulation of the sensitivity of the neural input-output response curve $\sigma(x)$ is typically observed in cortical neurons \citep{ferguson_mechanisms_2020}. There exist a number of biophysical mechanisms which allow top-down inputs to modulate $\sigma(x)$ \citep[e.g.,][]{larkum_top-down_2004}. Assuming a simple linear-threshold neuron model with weights $\theta$, this yields the response $\sigma(x) = g(\theta \cdot x - b)_+$ to some input $x$, where $(\cdot)_+$ denotes the positive-part operation. In this model, there are only few learnable parameters $\phi = \{g,b\}$, as they scale with the number of neurons and not with the number of synaptic connections.

We apply contrastive meta-learning to this model by changing synaptic weights $\theta$ according to our rule \eqref{eq:CML-delta-theta}. For this model, partial derivatives of the augmented loss function correspond to the usual derivatives with respect to model parameters that are routinely evaluated to learn deep neural networks; our rule simply asks to compute them twice. We therefore build upon existing theories of learning by backpropagation-of-error in the brain and assume that some mechanism for neuron-specific spatial error backpropagation is available, for example via prediction error neural subpopulations \citep{whittington_approximation_2017} or dendritic error representations \citep{sacramento_dendritic_2018,richards_dendritic_2019,payeur_burst-dependent_2021}, or by invoking equilibrium propagation again \citep{scellier_equilibrium_2017}.

%%%%%%%%%%%%%%%%%%%%%%%%%%%%%%%%%%%%%%%%%%%%%%%%%%%%%%%%%%%%
\section{Theoretical and experimental analyses}
In the following, we theoretically analyze the approximation error incurred by our contrastive meta-learning rule before empirically testing it on a suite of meta-learning problems. The objective of our experiments is twofold. First, we aim to confirm our theoretical results and demonstrate the performance of contrastive meta-learning on standard machine learning benchmarks. Second, we want to illustrate the generality of our approach by applying it to various supervised and reinforcement meta-learning problems as well as to a more biologically realistic neuron and plasticity model.

\subsection{Theoretical analysis of the meta-gradient approximation error}

The contrastive meta-learning rule \eqref{eq:CML-delta-theta} only provides an approximation to the meta-gradient. This approximation can be improved by refining the two learning solutions $\hat{\phi}_{\theta, 0, \tau}$ and $\hat{\phi}_{\theta, \beta, \tau}$ through additional computation or by using a better learning algorithm, and by decreasing the nudging strength $\beta$, as prescribed by the equilibrium propagation theorem. In Theorem~\ref{thm:bound_ep_estimate}, we theoretically analyze how the meta-gradient estimate \eqref{eq:CML-delta-theta} benefits from such improvements (see Fig.~\ref{fig:theory-vs-CIFAR}A for a visualization of the result, Section~\ref{sec_app:theory} for a proof and empirical verification of our theoretical results). We find that the refinement of the learning solutions must be coupled to a decrease in $\beta$: too small $\beta$ greatly detracts from the quality of the meta-gradient estimate when the solutions are not improved, while better approximations are inefficient if $\beta$ is not decreased accordingly.
\begin{theorem}[Informal]
    \label{thm:bound_ep_estimate}
    Let $\beta > 0$ and $\delta$ be such that $\lVert \hat{\phi}_{\theta, 0, \tau}-\phi_{\theta, 0, \tau}^* \rVert \leq \delta$ and $\lVert \hat{\phi}_{\theta, \beta, \tau}-\phi_{\theta, \beta, \tau}^* \rVert \leq \delta$. Then, under regularity and convexity assumptions, there exists a constant $C$ such that
    \begin{equation*}
        \lVert {-\Delta_{\theta, \tau} - \nabla_{\theta, \tau}} \rVert \leq C \left (\frac{1+\beta}{\beta}\,\delta + \frac{\beta}{1+\beta} \right) =: \mathcal{B}(\delta, \beta).
    \end{equation*}
\end{theorem}

\subsection{Contrastive meta-learning is a high-performance meta-optimization algorithm}
\label{sect:hyperparam_opt}
As a first set of experiments, we study a supervised
meta-optimization problem based on the entire CIFAR-10 image dataset \citep{krizhevsky_learning_2009}. In these experiments the goal is to meta-learn a set of hyperparameters (meta-parameters) such that generalization performance improves. This problem is a common testbed for assessing the ability of a meta-learning algorithm to optimize a given meta-objective \citep{lorraine_optimizing_2020}; it can be thought of as a limiting case of full meta-learning, as there are learnable meta-parameters, but only one task. As the meta-objective we take the cross-entropy loss $l$ evaluated on a held-out dataset $D^\mathrm{eval}$: $L^\mathrm{eval}(\phi) = \frac{1}{|D^\mathrm{eval}|} \, \sum_{(x,y)\in D^\mathrm{eval}} l(x,y,\phi)$, where $x$ is an image input and $y$ its label. We equip a convolutional deep neural network with our synaptic model~\eqref{eq:synaptic-model-learning-loss}, meta-learning only the per-synapse regularization strength $\lambda$, keeping $\omega$ fixed at zero: $L^\mathrm{learn}(\phi,\lambda) = \frac{1}{|D^\mathrm{learn}|} \, \sum_{(x,y)\in D^\mathrm{learn}} l(x,y,\phi) + \frac{1}{2}\sum_{i=1}^{|\phi|} \lambda_i \phi_i^2$. We learn the weights $\phi$ by stochastic gradient descent paired with backpropagation. Additional details and analyses may be found in Section~\ref{sec_app:supervised_metaopt}.

\begin{wraptable}[13]{r}{0.5\textwidth}
  \vspace{-0.7cm}
  \caption{Meta-learning a per-synapse regularization strength meta-parameter (cf.~Section~\ref{sect:synaptic-model}) on CIFAR-10. Average accuracies (acc.) $\pm$ s.e.m.~over 10 seeds.}
  \label{tab:cifar}
  \centering
  \vspace{0.1cm}
  \begin{tabular}{lll}
      \toprule
      Method  & Evaluation acc. (\%)  & Test acc. (\%)
      \\
      \midrule
      T1-T2 & 64.77$^{\pm 0.40}$ & 62.57$^{\pm 0.31}$\\
      CG & 57.65$^{\pm 1.51}$ & 57.51$^{\pm 0.98}$\\
      RBP & 64.92$^{\pm 1.32}$ & 62.14$^{\pm 0.97}$\\
      CML & 74.43$^{\pm 0.53}$ & 66.94$^{\pm 0.25}$\\  \midrule
      No meta & 60.06$^{\pm 0.37}$ & 60.13$^{\pm0.38}$ \\
      TBPTL & 73.17$^{\pm 0.27}$ & 65.35$^{\pm 0.36}$\\
      \bottomrule
  \end{tabular}
  \end{wraptable}
We benchmark our meta-plasticity rule \eqref{eq:synaptic-model-updates} against implicit gradient-based meta-learning methods, which are considered state-of-the-art for this type of problem \citep{lorraine_optimizing_2020} (see Section~\ref{sec_app:review_implicit} for a review). More concretely, recurrent backpropagation (RBP \citep{almeida_backpropagation_1989,pineda_recurrent_1989}; also known as the Neumann series approximation \citep{liao_reviving_2018,lorraine_optimizing_2020}) and the conjugate gradient method (CG) \citep{foo_efficient_2007,pedregosa_hyperparameter_2016} correspond to two different numerical schemes for calculating the meta-gradient; T1-T2 \citep{luketina_scalable_2016} is an approximate method which neglects complicated terms, thus introducing a non-reducible bias in the meta-gradient estimate. Critically, unlike our contrastive meta-learning rule (CML), this method offers no control over the meta-gradient error.

\begin{figure}[t]
    \centering
    \includegraphics[width=4.5in]{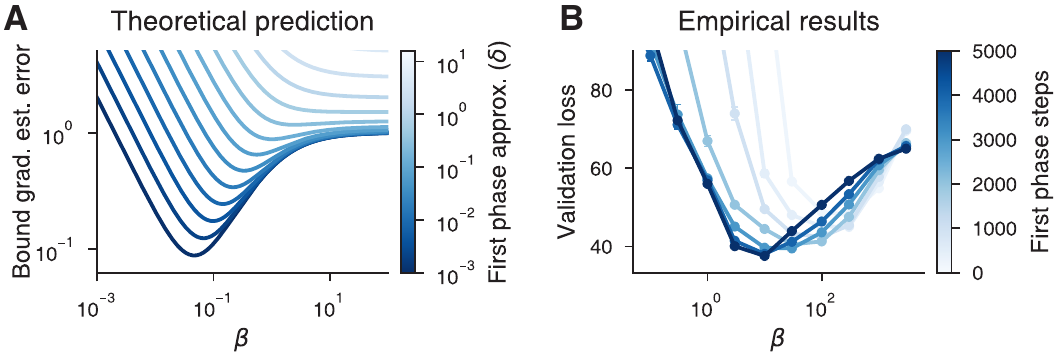}
    \vspace{-0.3cm}
    \caption{(A) Visualization of the theoretical bound $\mathcal{B}$ on the meta-gradient estimation error from Theorem~\ref{thm:bound_ep_estimate} as a function of the nudging strength $\beta$. Better approximations of the solutions (smaller $\delta$) improve the quality of the meta-gradient, as they enable using smaller values of $\beta$. (B) Confirmation of the qualitative findings of the theory on deep learning experiments. We show results for a hyperparameter meta-learning problem, where a per-synapse regularization strength is meta-learned (cf.~Section~\ref{sect:synaptic-model}) on CIFAR-10 with rule~\eqref{eq:synaptic-model-updates}. The validation loss is a proxy for the quality of the gradient and the number of steps in the first phase is a proxy for $-\log \delta$.\label{fig:theory-vs-CIFAR}}
\end{figure}

We find that our meta-learning rule outperforms all three baseline implicit differentiation methods in terms of both evaluation-set and actual generalization (test-set) performance, cf.~Tab.~\ref{tab:cifar}.
As a side result, we confirm the instability of CG in deep learning reported in ref.~\citep{liao_reviving_2018,shaban_truncated_2019}. We note that the hyperparameters of all four methods were independently and carefully set (cf.~Section~\ref{sec_app:supervised_metaopt}). These strong results on a modern deep learning benchmark, involving stochastic approximate learning, demonstrate that contrastive meta-learning is a scalable, highly effective meta-optimization algorithm. Moreover, Theorem~\ref{thm:bound_ep_estimate} is in excellent qualitative agreement with our experiments, cf.~Fig.~\ref{fig:theory-vs-CIFAR}.

To further contextualize our findings, we provide results for training the same network without meta-learning, where we performed a conventional hyperparameter search over a scalar regularization strength hyperparameter shared by all synapses. This simple approach yields only a moderate evaluation and test accuracy.

As all methods incur numerical errors when computing the meta-gradient, a comparison to using the analytical solution for the meta-gradient would be desirable. Since this is intractable in this case and running full backpropagation-through-learning requires too much memory, we evaluate truncated backpropagation-through-learning (TBPTL) with the maximal truncation window we can fit on a single graphics processing unit (in our case 200 out of 5000 steps). The resulting evaluation accuracy and test accuracy outperform other implicit gradient-based meta-learning methods but are still surpassed by our method.

%To conclude, our contrastive meta-learning rule is a powerful meta-optimizer outperforming even TBPTL in this setting despite its significantly lower memory costs.

\subsection{Contrastive meta-learning enables visual few-shot learning}

The ability to learn new object classes based on only a few examples is a hallmark of human intelligence \citep{lake_human-level_2015} and a prime application of meta-learning. We test whether our contrastive meta-learning rule is able to turn into a few-shot learner a standard visual system, a convolutional deep neural network learned by gradient descent and error backpropagation. Furthermore, we ask how our contrastive meta-learning rule fares against other gradient-based meta-learning algorithms which rely on backpropagation-through-learning and implicit differentiation to compute gradients. To that end, we focus on two widely-studied few-shot image classification problems based on miniImageNet \citep{ravi_optimization_2016} and the Omniglot \citep{lake_one_2011} datasets. To further facilitate comparisons, we reproduce exactly the experimental setup of ref.~\citep{finn_model-agnostic_2017}, which has been adopted in a large number of studies.

Briefly, during meta-learning, $N$-way $K$-shot tasks are created on-the-fly by sampling $N$ classes at random from a fixed pool of classes, and then splitting the data into task-specific learning $D_\tau^\mathrm{learn}$ (with $K$ examples per class for learning) and evaluation $D_\tau^\mathrm{eval}$ sets, used to define the corresponding loss functions $L_\tau^\mathrm{learn}$ and $L_\tau^\mathrm{eval}$. The meta-objective is then simply the task-averaged evaluation loss, measured after learning. The performance of the learning algorithm is tested on new tasks consisting of classes that were not seen during meta-learning. We provide all experimental details in Section~\ref{sec_app:fewshot}.

\begin{wraptable}[12]{r}{0.4\textwidth}
\vspace{-0.62cm}
\caption{One-shot miniImageNet learning. Averages over 5 seeds $\pm$ std.}
\label{tab:miniimagenet}
\centering
\vspace{0.1cm}
\begin{tabular}{ll}
    \toprule
    Method    & Test acc. (\%)
    \\
    \midrule
    MAML \cite{finn_model-agnostic_2017} & 48.70 $^{\pm 1.84}$\\
    FOMAML \cite{finn_model-agnostic_2017}  & 48.07 $^{\pm  1.75}$\\
    Reptile \cite{nichol_first-order_2018} & 49.97 $^{\pm  0.32}$ \\
    \midrule
    iMAML \cite{rajeswaran_meta-learning_2019} & 48.96 $^{\pm 1.84}$\\
    \midrule
    CML (synaptic) &  48.43 $^{\pm 0.43}$ \\
    CML (modulatory)  & 49.80 $^{\pm 0.40}$\\
    \bottomrule
  \end{tabular}
\end{wraptable}
As reference methods, we compare against the well-known model-agnostic meta-learning (MAML) algorithm \citep{finn_model-agnostic_2017}, which relies on backpropagation-through-learning to meta-learn an initial set of weights, starting from which a few gradient steps should succeed; this is conceptually similar to meta-learning the consolidated state $\omega$ of our complex synapses. We also include results obtained with its first-order approximation FOMAML (as well as a closely related algorithm known as Reptile \citep{nichol_first-order_2018}), which, like the T1-T2 algorithm of the previous section, excludes all second-order terms from the meta-gradient estimate to simplify the update, at the expense of introducing a bias. Finally, we compare to the implicit MAML (iMAML) algorithm \citep{rajeswaran_meta-learning_2019}, which corresponds exactly to meta-learning our consolidated synaptic state $\omega$, but with implicit differentiation methods.

\begin{wraptable}[12]{r}{0.58\textwidth}
\vspace{-0.62cm}
\caption{Omniglot character few-shot learning. Test set classification accuracy (\%) averaged over 5 seeds $\pm$ std.}
\label{tab:omniglot}
\centering
\vspace{0.05cm}
\begin{tabular}{lll}
    \toprule
    Method    
              & 20-way 1-shot & 20-way 5-shot
    \\
    \midrule
    MAML \cite{finn_model-agnostic_2017} & 95.8$^{\pm0.3}$ & 98.9$^{\pm0.2}$\\
    FOMAML \cite{finn_model-agnostic_2017}  & 89.4$^{\pm0.5}$ & 97.9$^{\pm0.1}$\\
    Reptile \cite{nichol_first-order_2018} & 89.43$^{\pm0.14}$ & 97.12$^{\pm0.32}$\\
    \midrule
    iMAML \cite{rajeswaran_meta-learning_2019} &94.46$^{\pm0.42}$ &98.69$^{\pm0.1}$\\
    \midrule
    CML (synaptic) & 94.16$^{\pm0.12}$ & 98.06$^{\pm0.26}$\\
    CML (modulatory)  & 94.24$^{\pm0.39}$  & 98.60$^{\pm0.27}$  \\
    \bottomrule
  \end{tabular}
\end{wraptable}
When applied to the problem domain of miniImageNet one-shot learning tasks, the performance of all meta-learning algorithms we consider here is closely clustered together, cf.~Tab.~\ref{tab:miniimagenet}. In particular, meta-learning the consolidated states $\omega$ of our complex synapses with implicit differentiation (iMAML) or our local update \eqref{eq:synaptic-model-updates} leads to comparable performance. Interestingly, we further find that miniImageNet one-shot learning performance is significantly improved when using the modulatory model described in~Section~\ref{sect:modulation}, despite the low dimensionality of the task-specific variable $\phi$. This is in line with other results suggesting that highly efficient visual learning of new categories may be possible without necessarily engaging synaptic plasticity \citep{zintgraf_fast_2019}. On Omniglot (see Section~\ref{sec_app:fewshot} for additional variants), the situation is comparable, except that on its 20-way 1-shot variant, the performance gap between first- and second-order methods widens. In line with our theory, our contrastive meta-learning rule performs close to (second-order) implicit differentiation, showing that despite its simplicity and locality our rule is able to accurately estimate meta-gradients.

\subsection{Contrastive meta-learning enables meta-plasticity in a recurrent spiking network}

For the experiments described on the previous sections we used simple artificial neuron models and backpropagation-of-error to learn. We now move closer to a biological neuron and plasticity model and consider meta-learning in a recurrently-connected neural network of leaky integrate-and-fire neurons with plastic synapses. We study a simple few-shot regression problem \citep{finn_model-agnostic_2017}, where the aim is to quickly learn to approximate sinusoidal functions which differ in their phase and amplitude (for additional details see Section \ref{sec_app:spiking}). For each task, we measure the mean squared error on 10 samples for the learning loss and 10 samples for the evaluation loss. We implement synaptic plasticity using the local e-prop rule \citep{bellec_solution_2020} and use a population of 100 Poisson neurons to encode inputs, see Fig.~\ref{fig:spiking}A. As our contrastive meta-learning rule~\eqref{eq:CML-delta-theta} is agnostic to the specifics of the learning process, we can augment the model with our synaptic consolidation model and apply the meta-plasticity rules derived in~\eqref{eq:synaptic-model-updates}.
Fig.~\ref{fig:spiking}B illustrates how the learning process improves with increasing number of tasks encountered, eventually consolidating a sinusoidal prior that can be quickly adapted to the specifics of a task from few examples, cf.~Fig.~\ref{fig:spiking}C.
\begin{wraptable}[11]{r}{0.52\textwidth}
%\vspace{-0.5cm}
\caption{Few-shot learning of sinusoidal functions with a recurrent spiking neural network. Avg.~mean squared error (MSE) over 10 seeds $\pm$ s.e.m.}
\label{tab:spiking}
\centering
\begin{tabular}{lll}
\toprule
Method        & Validation MSE    & Test MSE          \\
\midrule
BPTL + BPTT   & 0.17$^{\pm 0.01}$ & 0.41$^{\pm 0.10}$ \\
BPTL + e-prop & 0.52$^{\pm 0.05}$ & 0.72$^{\pm 0.08}$ \\
TBPTL + e-prop & 0.27 $^{\pm 0.07}$ & 0.50 $^{\pm 0.11}$ \\
CML + e-prop  & 0.23$^{\pm 0.04}$ & 0.23$^{\pm 0.04}$ \\
\bottomrule
\end{tabular}
\end{wraptable}

We compare our method to a standard baseline where updates are computed by backpropagating through the synaptic plasticity process (backpropagation-through-learning; BPTL) using surrogate gradients to handle spiking nonlinearities \citep{neftci_surrogate_2019} similar to previous work on spiking neuron meta-learning \citep{bellec_long_2018}. Since full BPTL requires reducing the number of learning steps compared to our method due to memory constraints, we also include TBPTL with the same number of 500 learning steps and a truncation window of 100 steps. In both cases, we find competitive performance for our method, see Tab.~\ref{tab:spiking}.

\begin{figure}[h!]
    \centering
    \includegraphics[width=5.5in]{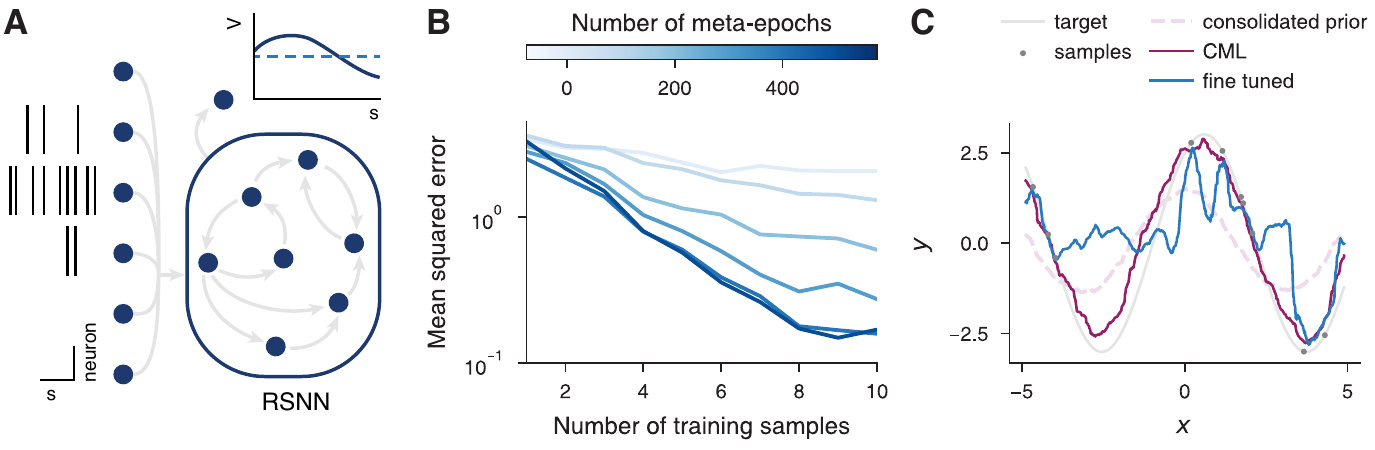}
    \caption{(A) A network of recurrently-connected leaky-integrate and fire neurons is tasked with learning sinusoids on an input encoding of Poisson spike trains. Its prediction is the voltage of the output neuron averaged over time. (B) Learning performance from few examples measured as the mean squared error on evaluation examples during a learning episode improves as more tasks are encountered over the course of meta-learning. (C) Meta-plasticity encodes information on the consolidated synaptic component (dashed) which results in improved learning performance (purple), compared to a naive network learning from scratch (blue).}
    \label{fig:spiking}
    \vspace{-0.5cm}
\end{figure}

\subsection{Contrastive meta-learning improves reward-based learning}

Finally, we demonstrate how contrastive meta-learning can be applied in the challenging setting of reward-based learning, second nature to most animals.
%In the brain, reward-based learning is classically thought to happen at the interface between the cortex and striatum with cortical inputs modifying reward predictions by providing a context signal \citep{collins_opponent_2011, gerfen_modulation_2011, mikhael_learning_2016}.
% It has recently been suggested that the role of the cortex is more complex and extends from supporting hierarchical and model-based learning \citep{banerjee_reinforcement-guided_2021} to maintaining its own context-dependent reinforcement learning algorithm \citep{wang_prefrontal_2018}.
Reward-based learning clearly demonstrates hallmarks of meta-learning as animals are capable of flexibly remapping reward representations when task contingencies change \citep{banerjee_value-guided_2020, samborska_complementary_2021}.
Inspired by this, we aim to meta-learn a value function on a family of reward-based learning tasks that can be quickly adapted to predict the expected reward of the actions available to the agent in a particular task.

Specifically, we consider the wheel bandit problem introduced by \citep{riquelme_deep_2018} with the meta-learning setup previously studied in refs.~\citep{garnelo_neural_2018,ravi_amortized_2019}. On each task, an agent is presented with a sequence of context coordinates randomly drawn from a unit circle for each of which it has to choose among 5 actions to receive a stochastic reward. Hidden to the agent, a task-specific radius $\delta$ tiles the context space into a low- and a high-reward region depending on which the optimal action to take changes (see Section \ref{sec_app:metaRL}).

\begin{wraptable}[11]{r}{0.62\textwidth}
\vspace{-0.6cm}
\caption{Cumulative regret on the wheel bandit problem for different $\delta$. Values normalized by the cumulative regret of a uniformly random agent. Avgs. over 50 seeds $\pm$ s.e.m.}
\label{tab:bandit}
\centering
\vspace{0.15cm}
\begin{tabular}{llll}
\toprule
$\delta$         & 0.5               & 0.9               & 0.99               \\
\midrule
NeuralLinear  \citep{riquelme_deep_2018} & 0.95$^{\pm 0.02}$ & 4.65$^{\pm 0.18}$ & 49.63$^{\pm 2.41}$ \\
\midrule
MAML             & 0.45$^{\pm 0.01}$ & 1.02$^{\pm 0.76}$ & 15.21$^{\pm 1.69}$ \\
CML (synaptic)   & 0.40$^{\pm 0.02}$ & 0.82$^{\pm 0.02}$ & 12.27$^{\pm 1.02}$ \\
CML (modulatory) & 0.42$^{\pm 0.01}$ & 1.83$^{\pm 0.11}$ & 16.46$^{\pm 1.80}$ \\
\bottomrule
\end{tabular}
\end{wraptable}
The goal of meta-learning is to discern the general structure of the low- and high-reward region across tasks whereas the goal of learning becomes to identify the task-specific radius $\delta$ of the current task.
During meta-learning, we randomly sample tasks $\delta \sim \mathcal{U}(0,1)$ and generate a dataset by choosing actions randomly. Data from each task is split into training and evaluation data, effectively creating a sparse regression problem where only the outcome of a randomly chosen action can be observed for a particular context.
After meta-learning, we evaluate the cumulative regret obtained by an agent that chooses his actions greedily with respect to its predicted rewards and adapts its fast parameters on the observed context, action, reward triplets stored in a replay buffer. 

We use both our synaptic consolidation and modulatory network models to meta-learn the value function using our contrastive rule. We compare our two models to MAML and the non-meta-learned baseline, NeuralLinear, from ref.~\citep{riquelme_deep_2018}, which performed among the best in their large-scale comparison. Tab.~\ref{tab:bandit} shows the cumulative regret obtained on different task parametrizations $\delta$ in the online evaluation after meta-learning (extended table in Section~\ref{sec_app:metaRL}). Meta-learning clearly improves upon the non-meta-learned baseline with both our models performing comparably to MAML. This improvement is more pronounced for tasks with larger $\delta$ within which it is more difficult to discover the high-reward region.

%%%%%%%%%%%%%%%%%%%%%%%%%%%%%%%%%%%%%%%%%%%%%%%%%%%%%%%%%%%%

\section{Discussion}
We have presented a general-purpose meta-learning rule which allows estimating meta-gradients from local information only, and we have demonstrated its versatility studying two neural models on a range of meta-learning problems. The competitive performance we observed suggests that contrastive meta-learning is a worthy contender to biologically-implausible machine learning algorithms -- especially for problems involving long learning trajectories, as demonstrated by the strong results on supervised meta-optimization.
At its core, our method relies on contrasting the outcome of two different learning episodes. Despite its conceptual simplicity this requires complex synaptic machinery which is able to buffer these outcomes in a way accessible to synaptic consolidation.

According to our top-down modulation model the goal of synaptic plasticity in primary brain areas is \emph{not} to learn a specific task, in contrast to more traditional theories of learning. Instead, we postulate that the goal of synaptic plasticity is to make it possible to learn any given task by modulating the sensitivity of primary-area neurons in a task-dependent manner. This view is consistent with the experimental findings of \citet{fritz_adaptive_2010}, who observed the rapid formation of task-dependent receptive fields in the primary auditory cortex of ferrets, as the animals learned several tasks, presumably due to changes in top-down signals originating in frontal cortex. Together with the strong results of the modulatory model in the challenging setting of visual one-shot learning and recent studies in continual learning problems \citep{masse_alleviating_2018,wen_batchensemble_2020,von_oswald_continual_2020,tsuda_modeling_2020} this shows the practical effectiveness of task-dependent modulation. 
Complementary to the interaction of the frontal cortex with primary cortical areas, the prefrontal cortex might similarly modulate the striatum during reward-based learning. Whereas classical dopamine-based learning posits that reward prediction errors are used subcortically to learn the reward structure of a task, recent work has demonstrated that reward can similarly affect prefrontal representations to quickly infer the current task identity and switch the context provided to the striatum \citep{blanco-pozo_dopamine_2021}.
More broadly, viewing synaptic plasticity as meta-learning is also consistent with recent modeling work casting the prefrontal cortex as a meta-reinforcement learning system \citep{wang_prefrontal_2018}.

Reflecting on how our meta-learning rule can be implemented in the brain, we conjecture that the hippocampal formation plays a central role in coordinating the two phases as well as creating the augmented learning problem.
First, some mechanism must signal that a switch from learning problem to augmented learning problem has occurred, corresponding to the sign switch in our rule \eqref{eq:CML-delta-theta}. We argue that the hippocampus is well positioned for signaling such a switch to cortical synapses. A recent experimental study shows that the hippocampus is at least able to control cortical synaptic consolidation \cite{doron_perirhinal_2020} but further evidence would be needed to support our hypothesis.

Second, we conjecture that the creation of the augmented learning problem at the heart of our meta-gradient estimation algorithm might itself critically rely on the hippocampus. In all our experiments, this second learning problem consisted simply of new data, presented to the learning algorithm to evaluate how well learning went. Transferring additional data into cortical networks, putatively during sleep and wakeful rest, fits well with the role that is classically attributed to the hippocampus in systems consolidation and complementary learning systems theories \citep{mcclelland_why_1995,kumaran_what_2016}. We thus speculate that the hippocampus `prescribes' additional learning problems to the cortex, which serve the purpose of testing its generalization performance. By showing that a second `sleep' learning phase enables meta-learning with simple plasticity rules, our results lend further credit to complementary learning systems theory, as well as to the hypothesis that dreams have evolved to assist generalization \citep{hoel_overfitted_2021}.

Lastly, this view of the cortex as a contrastive meta-learning system aided by the hippocampus may also help elucidate how the brain learns from an endless, non-stationary stream of data. Current artificial neural networks notoriously struggle to strike a balance between learning new knowledge and retaining old one in such continual learning problems, in particular when the data are not independent and identically distributed nor structured into clearly delineated tasks \citep{hadsell_embracing_2020}. Interestingly, recent investigations have shown that meta-learning can greatly improve continual learning performance \citep{javed_meta-learning_2019,riemer_learning_2019,gupta_look-ahead_2020,beaulieu_learning_2020,von_oswald_learning_2021}. While details vary, the essence of these methods is to blend in past (replay) data with new data in a meta-objective function. This amounts to a different instantiation of our bilevel optimization problem \eqref{eq:bilevel-meta-learning-problem}, resulting in an augmented learning problem in which past and present data are intermixed, for which the hippocampus would again appear to be ideally positioned.

\vspace{-0.5cm}
\begin{ack}
\vspace{-0.5cm}
This research was supported by an Ambizione grant (PZ00P3\_186027) from the Swiss National Science Foundation and an ETH Research Grant (ETH-23 21-1) awarded to João Sacramento. Johannes von Oswald is funded by the Swiss Data Science Center (J.v.O. P18-03). We thank Angelika Steger, Benjamin Scellier, Greg Wayne, Abhishek Banerjee, Blake A.~Richards, Nicol Harper, Thomas Akam, Mohamady El-Gaby, Rafal Bogacz, Giacomo Indiveri, Jean-Pascal Pfister, Mark van Rossum, Maciej Wołczyk, Seijin Kobayashi and Alexander Meulemans for discussions and feedback, and Charlotte Frenkel for assistance in our implementation of e-prop.
\end{ack}

\bibliographystyle{unsrtnat}
\bibliography{references}

\newpage
\appendix

\setcounter{page}{1}
\setcounter{figure}{0} \renewcommand{\thefigure}{S\arabic{figure}}
\setcounter{theorem}{0} \renewcommand{\thetheorem}{S\arabic{theorem}}
\setcounter{definition}{0} \renewcommand{\thedefinition}{S\arabic{definition}}
\setcounter{section}{0}
\renewcommand{\thesection}{S\arabic{section}}
\setcounter{table}{0} \renewcommand{\thetable}{S\arabic{table}}

\renewcommand \thepart{}
\renewcommand \partname{}
\renewcommand \thepart{Supplementary Materials}
% \tableofcontents
\addcontentsline{toc}{section}{} 
\part{} % Start the appendix part
\textbf{Nicolas Zucchet$^*$, Simon Schug$^*$, Johannes von Oswald$^*$, Dominic Zhao, João Sacramento}
\parttoc % Insert the appendix TOC

\section{Derivation of the contrastive meta-learning rule}
\label{sec_app:equilibrium_propagation}

Our contrastive meta-learning rule relies on the equilibrium propagation theorem \citesupp{scellier_equilibrium_2017, scellier_deep_2021}. We review this result and how we use it to derive the different instances of our rule.

\subsection{Equilibrium propagation theorem}

First, we  restate the equilibrium propagation theorem as presented in \citet{scellier_deep_2021}. Recall the definition of the augmented loss
\begin{equation}
    \Lt(\phi, \theta, \beta) = \Li(\phi, \theta) + \beta \Lo(\phi,\theta).
\end{equation}
Note that compared to the main text, we omit the subscript $\tau$ for conciseness.
Given the augmented loss, the equilibrium propagation theorem states the following:

\begin{theorem}[Equilibrium propagation]
    \label{thm:eq_prop}
    Let $\Li$ and $\Lo$ be two twice continuously differentiable functions. Let $\phi^*$ be a fixed point of $\Lt(\,\cdot\,, \bar{\theta}, \bar{\beta})$, i.e.
    \begin{equation*}
        \pder{\phi}{\Lt}(\phi^*, \bar{\theta}, \bar{\beta}) = 0,
    \end{equation*}
    such that $\partial_\phi^2 \Lt(\phi^*, \bar{\theta}, \bar{\beta})$ is invertible. Then, there exists a neighborhood of $(\bar{\theta}, \bar{\beta})$ and a continuously differentiable function $(\theta, \beta) \mapsto \phisbt$ such that $\phi^*_{\bar{\theta}, \bar{\beta}} = \phi^*$ and for every $(\theta, \beta)$ in this neighborhood
    \begin{equation*}
        \pder{\phi}{\Lt}(\phisbt, \theta, \beta) = 0.
    \end{equation*}
    Furthermore, 
    \begin{equation*}
        \frac{\mathrm{d}}{\mathrm{d}\theta}\pder{\beta}{\Lt}\left(\phisbt, \theta, \beta \right) =  \frac{\mathrm{d}}{\mathrm{d}\beta}\frac{\partial \Lt}{\partial \theta}\left(\phisbt, \theta,\beta \right)^\top\!.
    \end{equation*}
\end{theorem}

\begin{proof}
    The first point follows from the implicit function theorem \citesupp{dontchev_implicit_2009}. Let $(\theta, \beta)$ be in a neighborhood of $(\bar{\theta}, \bar{\beta})$ in which $\phisbt$ is differentiable.
    
    The symmetry of second order derivatives of a scalar function implies that
    \begin{equation}
        \der{\theta}{}\der{\beta}{}\Lt\left(\phisbt, \theta, \beta\right) =  \der{\beta}{}\der{\theta}{}\Lt\left(\phisbt, \theta, \beta\right)^\top \!.
    \end{equation}
    We then simplify the two sides of the equation. First, we look at the left-hand side and simplify $\mathrm{d}_\beta\Lt(\phisbt, \theta, \beta)$ using the chain rule and the fixed point condition
    \begin{equation}
        \begin{split}
            \der{\beta}{}\Lt(\phisbt, \theta, \beta) &= \pder{\beta}{\Lt}(\phisbt, \theta, \beta)+\pder{\phi}{\Lt}(\phisbt, \theta, \beta)\der{\beta}{\phisbt}\\
            & = \pder{\beta}{\Lt}(\phisbt, \theta, \beta).
        \end{split}
    \end{equation}
    Similarly, the $\mathrm{d}_\theta\Lt(\phisbt,\theta,\beta)$ term on the right-hand side is equal to $\partial_\theta\Lt(\phisbt,\theta,\beta)$ and we obtain the required result\footnote{Note that we use $\partial$ to denote partial derivatives and $\mathrm{d}$ to denote total derivatives.}.
\end{proof}

\subsection{The contrastive meta-learning rule}

Equilibrium propagation can be used to compute the gradient associated with the bilevel optimization problem studied in this paper
\begin{equation}
    \label{eqn_app:bo_metalearning}
    \min_\theta \; \Lo(\phist) \quad \mathrm{s.t.} \enspace \phist \in \argmin_\phi \, \Li(\phi, \theta).
\end{equation}
To do so, we first characterize $\phist$ through the stationarity condition
\begin{equation}
    \pder{\phi}{\Li}(\phist, \theta) = 0.
\end{equation}
As $\Li(\phi, \theta) = \Lt(\phi, \theta, 0)$ we can define an implicit function $\phisbt$ if $\spder{\phi}{^2 \Li}(\phist, \theta)$ is invertible for which $\phiszt = \phist$, and that satisfies, for $\beta$ close to 0,
\begin{equation}
    \pder{\phi}{\Lt}(\phisbt, \theta, \beta) = 0.
\end{equation}
The gradient associated with the bilevel optimization problem \eqref{eqn_app:bo_metalearning} is then equal to
\begin{equation}
    \nabla_\theta := \left ( \der{\theta}{}\Lo(\phist, \theta) \right )^\top = \left . \der{\theta}{}\pder{\beta}{\Lt}(\phisbt, \theta, \beta) \right |_{\beta=0}^\top = \left . \der{\beta}{}\pder{\theta}{\Lt}(\phisbt, \theta, \beta) \right |_{\beta=0}.
\end{equation}

Since $\beta$ is a scalar, we can use finite difference methods to efficiently estimate
\begin{equation}
    \label{eqn_app:contrastive_update}
    \Delta \theta = -\widehat{\nabla}_\theta = -\frac{1}{\beta}\left ( \pder{\theta}{\Lt}(\phihb, \theta, \beta) - \pder{\theta}{\Lt}(\phihz, \theta, 0)\right )^\top,
\end{equation}
in which $\phihz$ and $\phihb$ denote the estimates of $\phiszt$ and $\phisbt$. If those estimates are exact, we are guaranteed that the update converges to the true gradient. In some of our experiments, we use a more sophisticated center difference approximation similar to \citesupp{laborieux_scaling_2021}, that is
\begin{equation}
    \label{eq:symmetric-update}
    \Delta \theta^\mathrm{sym} = -\widehat{\nabla}^\mathrm{sym}_\theta = -\frac{1}{2\beta}\left ( \pder{\theta}{\Lt}(\phihb, \theta, \beta) - \pder{\theta}{\Lt}(\hat{\phi}_{-\beta}, \theta, -\beta)\right )^\top.
\end{equation}
We refer to \eqref{eq:symmetric-update} as the symmetric variant of our contrastive rule.
When the estimates for the fixed points are exact, it reduces the meta-gradient estimation bias from $O(\beta)$ for the forward difference above to $O(\beta^2)$ at the expense of having to run a third phase.

\subsection{Application to the complex synapse model}

We can now derive the meta-learning rules for the complex synapse model of Section~\ref{sect:synaptic-model}. Recall that
\begin{equation}
    \Li(\phi, \theta) = l^\mathrm{learn}(\phi) + \frac{1}{2}\sum_{i=1}^{|\phi|}\lambda_i (\omega_i - \phi_i)^2
\end{equation}
and
\begin{equation}
    \Lo(\phi) = l^\mathrm{eval}(\phi),
\end{equation}
where $l^\mathrm{eval}(\phi)$ and $l^\mathrm{learn}(\phi)$ are two data-dependent loss functions.

For the complex synapse model, only the learning loss depends on the meta-parameters, hence $\spder{\theta}{\Lt} = \spder{\theta}{\Li}$ and
\begin{equation}
    \begin{split}
        \pder{\omega}{\Lt}(\phi, \theta, \beta) &= \lambda (\omega - \phi)\\
        \pder{\lambda}{\Lt}(\phi, \theta, \beta) &= \frac{1}{2}(\omega - \phi)^2,
    \end{split}
\end{equation}
where all the operations are carried out elementwise. Plugging the last equation in the contrastive update \eqref{eqn_app:contrastive_update} yields
\begin{equation}
    \begin{split}
        \Delta \omega &= -\frac{\lambda}{\beta}\left ((\omega - \phihb) - (\omega - \phihz) \right ) = \frac{\lambda}{\beta}\left( \phihb - \phihz \right)\\
        \Delta \lambda &= -\frac{1}{2\beta}\left ((\omega - \phihb)^2 - (\omega - \phihz)^2 \right ) = \frac{1}{2\beta}\left ((\omega - \phihz)^2 - (\omega - \phihb)^2 \right ).
    \end{split}
\end{equation}

\subsection{Application to the top-down modulation model}

The structure of the learning and evaluation losses for the top-down modulation model is the following:
\begin{equation}
    \begin{split}
        \Li(\phi, \theta) &= l^\mathrm{learn}(h(\phi, \theta))\\
        \Lo(\phi, \theta) &= l^\mathrm{eval}(h(\phi, \theta)),
    \end{split}
\end{equation}
where $l^\mathrm{learn}(\psi)$ and $l^\mathrm{eval}(\psi)$ are two data-driven losses that use learning and evaluation datasets to evaluate the performance of a network parametrized by $\psi$, and $h(\phi, \theta)$ produces the parameters $\psi$ by modulating a base network parametrized by $\theta$. Specifically, we modulate the rectified linear unit (ReLU) activation function for each neuron $i$ with a gain $g_i$ and shift $b_i$, $\sigma_\phi(x_i) = g_i((\theta \cdot x)_i - b_i)_+ $, with the gain and shift parameters of all neurons defining the fast parameters $ \phi= \{ g, b \} $ .
%In fact, this modulation can be integrated directly in the $\theta$ parameters: the modulation described above is similar to multiplying all the weights arriving to a neuron $i$ by $g_i$ and adding $b_i$ to the already existing bias used for this neuron.

Applying our contrastive update \eqref{eqn_app:contrastive_update} to this model we obtain the following equations:
\begin{equation}
    \label{eqn_app:modulatory_udpate}
    \begin{split}
        \Delta \theta &= -\frac{1}{\beta}\left ( \pder{\theta}{\Lt}(\phihb, \theta, \beta) - \pder{\theta}{\Lt}(\phihz, \theta, 0) \right )^\top\\
        & = -\frac{1}{\beta}\left ( \pder{\psi}{[l^\mathrm{learn}+\beta l^\mathrm{eval}]}(h(\phihb, \theta))\pder{\theta}{h}(\phihb, \theta) - \pder{\psi}{l^\mathrm{learn}}(h(\phihz, \theta))\pder{\theta}{h}(\phihz, \theta) \right )^\top.
    \end{split}
\end{equation}
Let us now decompose what this update means. The losses $l^\mathrm{learn}(\psi)$ and $[l^\mathrm{learn} + \beta l^\mathrm{eval}](\psi)$ measure the performance of a network parametrized by $\psi$ on the learning data, and on a weighted mix of learning and evaluation data. The derivatives $\spder{\psi}{l^\mathrm{learn}}$ and $\spder{\psi}{[l^\mathrm{learn}+\beta l^\mathrm{eval}]}$ can therefore be computed using the backpropagation-of-error algorithm, or any biologically plausible alternative to it. Those derivatives are then multiplied by $\spder{\theta}{h}$, which is a diagonal matrix as the modulation does not combine weights together, but only individually changes them. As a result, the update (\ref{eqn_app:modulatory_udpate}) contrasts two elementwise modulated gradients with respect to the weights.

\newpage
\section{Review of implicit gradient methods for meta-learning}
\label{sec_app:review_implicit}

The gradient associated with the bilevel optimization problem of Eq.~\ref{eqn_app:bo_metalearning} can be calculated analytically using the implicit function theorem \citesupp{dontchev_implicit_2009}. This insight forms the basis for implicit gradient methods for meta-learning which we shortly review in the following. We additionally provide a comparison of the computational and memory complexity of different meta-learning methods in Table~\ref{tab:comparison}.

As for the derivation of the contrastive meta-learning rule, we start by characterizing the implicit function $\phist$ of $\theta$ through its corresponding first-order stationarity condition
\begin{equation}
    \pder{\phi}{\Li}(\phist, \theta) = 0.
\end{equation}
Then, when the Hessian $\spder{\phi}{^2\Li}(\phist, \theta)$ is invertible, we have
\begin{equation}
    \label{eqn_app:implicit_gradient}
    \begin{split}
        \der{\theta}{}\Lo(\phist, \theta) &= \pder{\theta}{\Lo}(\phist, \theta) + \pder{\phi}{\Lo}\der{\theta}{\phist}\\
        &= \pder{\theta}{\Lo}(\phist, \theta) - \pder{\phi}{\Lo}\left ( \pder{\phi^2}{^2\Li}(\phist, \theta)\right)^{-1}\pder{\phi \partial\theta}{^2\Li}(\phist, \theta),
    \end{split}
\end{equation}
where in the first line we used the chain rule and in the second line the differentiation formula provided by the implicit function theorem \citesupp{dontchev_implicit_2009}.

In most practical applications, $\phi$ is high dimensional rendering the computation and inversion of the Hessian $\spder{\phi}{^2\Li}(\phist, \theta)$ intractable. In order to obtain a practical algorithm, implicit gradient methods numerically approximate the row vector
\begin{equation}
    \mu := -\pder{\phi}{\Lo}(\phist, \theta)\left ( \pder{\phi^2}{^2\Li}(\phist, \theta)\right)^{-1}.
\end{equation}

The simplest algorithm, T1-T2 \citesupp{luketina_scalable_2016}, replaces the inverse Hessian by the identity, i.e. \linebreak $\mu \approx \spder{\phi}{\Lo}(\phist, \theta)$, which yields an estimate relying only on first derivatives. 

The recurrent backpropagation algorithm \cite[RBP, ][]{almeida_backpropagation_1989, pineda_recurrent_1989, liao_reviving_2018}, also known as Neumann series approximation \citesupp{liao_reviving_2018, lorraine_optimizing_2020}, builds on the insight that $\mu$ is the solution of the linear system
\begin{equation}
    x \pder{\phi^2}{^2\Li}(\phist, \theta) = -\pder{\phi}{\Lo}(\phist, \theta).
\end{equation}
which can be solved via fixed point iteration.

Finally, $\mu$ can be seen as the solution of the optimization problem
\begin{equation}
    \min_x\, x \pder{\phi^2}{^2\Li}(\phist, \theta) x^\top + x \pder{\phi}{\Lo}(\phist, \theta)^\top
\end{equation}
when the Hessian of $\Li$ is positive definite. This optimization problem can be efficiently solved via the conjugate gradient method \citesupp{goutte_adaptive_1998, rajeswaran_meta-learning_2019}.  

The three algorithms described above provide different estimates for $\mu$ but all follow the same basic procedure: 
\begin{enumerate*}[label=(\arabic*)]
\item minimize the learning loss to approximate $\phist$;
\item estimate $\mu$; and
\item update the meta-parameters using \eqref{eqn_app:implicit_gradient} with the estimated $\mu$.

\end{enumerate*}

Compared to our contrastive meta-learning rule, these algorithms require a second phase that is completely different from the first one and which involves second derivatives (apart from the biased T1-T2). Additionally, as mentioned in Section~\ref{sect:hyperparam_opt}, the conjugate gradient method is faster in theory, but was reported to be unstable by several studies \citesupp{liao_reviving_2018, shaban_truncated_2019}. We confirm those findings in our experiments (cf. Section~\ref{sec_app:boston_comparison} and \ref{sec_app:supervised_metaopt}).

\begin{table}[!htb]
\caption{Comparison of computational and memory complexity of meta-learning methods.\\
$T$ denotes the number of steps in the base learning process and $K$ refers to steps taken in an algorithm-specific second phase. “HVP” abbreviates “Hessian-vector product” and “cross der. VP” denotes “cross derivative vector product”. An algorithm is “exact in the limit” if it computes the meta-gradient or can approximate it with arbitrary precision given enough compute. The algorithms compared in this table are contrastive meta-learning (CML), conjugate gradients (CG; used in iMAML), recurrent backpropagation (RBP), T1-T2, backpropagation-through-learning (BPTL; used in MAML), its truncated version (TBPTL) and its first-order version where all Hessians are replaced by the identity (FOBPTL; also known as FOMAML) and Reptile. The first four algorithms assume that the base learning process reaches an equilibrium, whereas the last four require no such assumption. * Reptile is not a general-purpose meta-learning method as it is restricted to meta-learn the initialization of the learning process.}
\label{tab:comparison}
\centering
\begin{tabular}{@{}llllllc@{}}
\toprule
Method & \multicolumn{2}{l}{\# gradients w.r.t.} & \multicolumn{2}{l}{\# 2nd-order terms} & Memory  & Exact in the limit \\
                        & $\phi$            & $\theta$            & HVP           & cross der. VP          &                          &                                     \\ \midrule
CML (ours)              & $T + K$             & 2                   & 0             & 0                      & $\mathcal{O}(|\phi| + |\theta|)$   & \cmark                                   \\
CG \cite{foo_efficient_2007, pedregosa_hyperparameter_2016, rajeswaran_meta-learning_2019} & $T + 1$             & 1                   & $K$             & 1                      & $\mathcal{O}(|\phi| + |\theta|)$   & \cmark                                   \\
RBP \cite{almeida_backpropagation_1989, pineda_recurrent_1989, liao_reviving_2018, lorraine_optimizing_2020}                    & $T + 1$             & 1                   & $K$             & 1                      & $\mathcal{O}(|\phi| + |\theta|)$   & \cmark                                   \\
T1-T2 \cite{luketina_scalable_2016}                  & $T + 1$             & 1                   & 0             & 1                      & $\mathcal{O}(|\phi| + |\theta|)$   & \xmark                                  \\ \midrule
BPTL \cite{finn_model-agnostic_2017}                   & $T + 1$             & $T + 1$               & $T$             & 0                      & $\mathcal{O}(T |\phi| + |\theta|)$ & \cmark                                   \\
TBPTL \cite{shaban_truncated_2019}                  & $T + 1$             & $K + 1$               & $K$             & 0                      & $\mathcal{O}(K |\phi| + |\theta|)$ & \xmark                                  \\
FOBPTL \cite{finn_model-agnostic_2017}                 & $T + 1$             & $T + 1$               & 0             & 0                      & $\mathcal{O}(T |\phi| + |\theta|)$ & \xmark                                  \\
Reptile* \cite{nichol_first-order_2018}               & $T$                 & 0                   & 0             & 0                      & $\mathcal{O}(|\phi|)$              & \xmark                                 \\ \bottomrule
\end{tabular}
\vspace{0.2cm}
\end{table}
\clearpage
\newpage
\section{Theoretical results}
\label{sec_app:theory}

The contrastive meta-learning rule \eqref{eqn_app:contrastive_update} only provides an approximation $\widehat{\nabla}_\theta$ to the meta-gradient $\nabla_\theta$ due to the limited precision of the fixed points and the finite difference estimator. We can in principle arbitrarily improve the approximation by spending more compute to refine the quality of the solutions $\phihz$ and $\phihb$ and decreasing the nudging strength $\beta$. The purpose of this section is to theoretically analyze the impact of such a refinement on the quality of the meta-gradient estimate. We state Theorem~\ref{thm:bound_ep_estimate} formally, present a corollary of this result, and verify that it holds experimentally.

\subsection{Meta-gradient estimation error bound}

We start by upper bounding the meta-gradient estimation error $\snorm{\widehat{\nabla}_\theta - \nabla_\theta}$, given the value of $\beta$ and the error made in the approximation of the solutions of the lower-level learning process. Two conflicting phenomena impact the estimation error. First, our meta-learning rule uses potentially inexact solutions. Second, the finite difference approximation of the $\beta$-derivative yields the so-called finite difference error. To study those two errors in more detail, we introduce \begin{equation*}
    \widehat{\nabla}_\theta^* :=  \frac{1}{\beta}\left ( \pder{\theta}{\Lt}(\phisbt, \theta, \beta)-\pder{\theta}{\Lt}(\phiszt, \theta, 0)\right )\!,
\end{equation*}
the contrastive estimate of the meta-gradient $\nabla_\theta$, but evaluated at the exact solutions $\phiszt$ and $\phisbt$ (recall that $\widehat{\nabla}_\theta$ has the same structure, but it is evaluated on the approximate solutions $\phihz$ and $\phihb$). Equipped with $\widehat{\nabla}_\theta^*$, we now have a way to quantify the two errors described above: $\snorm{\nabla_\theta - \widehat{\nabla}_\theta^*}$ measures the finite difference error and $\snorm{\widehat{\nabla}_\theta^* - \widehat{\nabla}_\theta}$ measures the solution approximation induced error, that is the consequence of the imperfect solutions. 

Informally, higher $\beta$ values will reduce the sensitivity to crude approximations to the lower-level solutions while increasing the finite difference error. Theorem \ref{thm:bound_ep_estimate} theoretically justifies this intuition under the idealized regime of strong convexity and smoothness defined in Assumption \ref{ass_app:regularity_inner}. This result holds for every rule induced by equilibrium propagation.

\begin{assumption}
    \label{ass_app:regularity_inner}
    Assume that $\Li$ and $\Lo$ are three-times continuously differentiable and that they, as functions of $\phi$, verify the following properties.
    \begin{itemize}
        \item[i.] $\partial_\theta\Li$ is $B^\mathrm{learn}$-Lipschitz and $\partial_\theta\Lo$ is $B^\mathrm{eval}$-Lipschitz.
        \item[ii.] $\Li$ and $\Lo$ are $L$-smooth and $\mu$-strongly convex.
        \item[iii.] their Hessians are $\rho$-Lipschitz.
        \item[iv.] $\partial_\phi\partial_\theta\Li$ and $\partial_\phi \partial_\theta\Lo$ are $\sigma$-Lipschitz.
    \end{itemize}
\end{assumption}

\begin{reptheorem}{thm:bound_ep_estimate}[Formal]
    Let $\beta > 0$ and $(\delta, \delta')$ be such that
    \begin{equation*}
        \snorm{\phiszt-\phihz} \leq \delta, \quad \mathrm{and} \quad \snorm{\phisbt-\phihb} \leq \delta'.
    \end{equation*}
    Under Assumption~\ref{ass_app:regularity_inner}, there exists a $\theta$-dependent constant $C$ such that
    \begin{equation*}
        \snorm{\nabla_\theta - \widehat{\nabla}_\theta} \leq \frac{B^\mathrm{learn}(\delta+\delta')}{\beta} + B^\mathrm{eval}\delta' + C \frac{\beta}{1+\beta} =: \mathcal{B}(\delta, \delta', \beta).
    \end{equation*}
    If we additionally assume that $\theta$ lies in a compact set, we can choose $C$ to be independent of $\theta$.
\end{reptheorem}

\begin{figure}[htbp]
    \centering
    \includegraphics[width=5.5in]{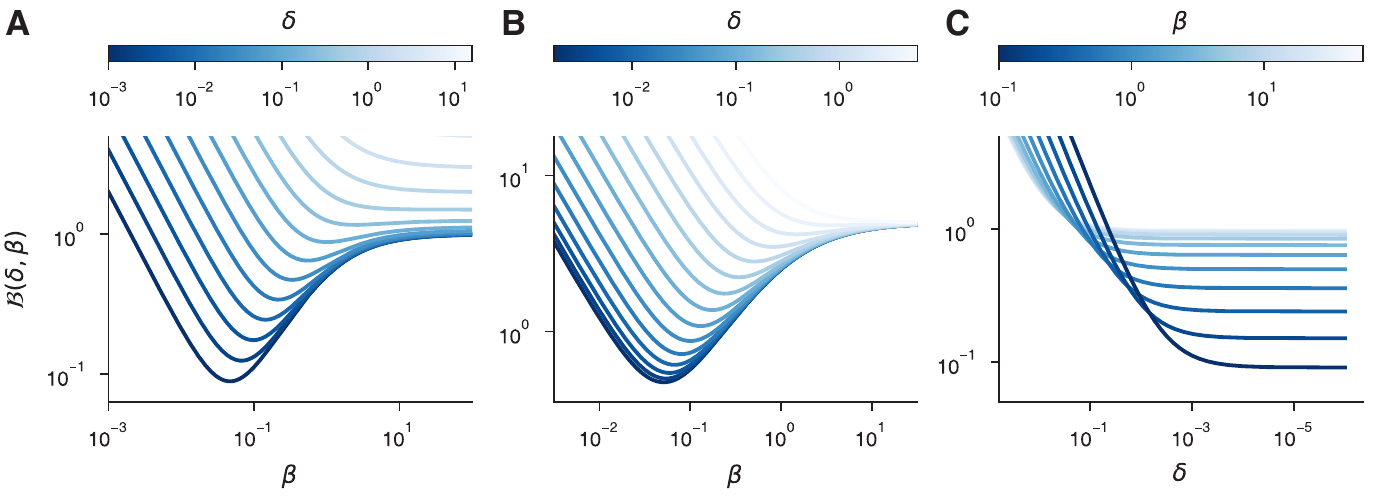}
    \caption{Theorem~\ref{thm:bound_ep_estimate} ($C=1$), as a function of $\beta$ (A, B) and as a function of $\delta=\delta'$ (C). (A) We take $B^\mathrm{learn} = B^\mathrm{eval} = 1$ and $\delta = \delta'$. (B) Bound for the setting in which $\delta'$ is fixed to $0.01$ and $\Lo$ is independent of $\theta$ (as for the complex synapse model). (C) We use the same setting as for (A).}
    \label{fig_app:bound_viz}
\end{figure}

We visualize our bound in Fig.~\ref{fig_app:bound_viz}, as a function of $\beta$ and of the solution approximation errors $\delta$ and $\delta'$. When $\delta$ and $\delta'$ are fixed, the estimation error quickly increases when $\beta$ deviates from its optimal value and it saturates for large $\beta$ values (cf.~Fig.~\ref{fig_app:bound_viz}A and B). A better solution approximation naturally improves the quality of the meta-gradient estimate for $\beta$ held constant (cf.~Fig.~\ref{fig_app:bound_viz}C). However, the benefits saturate above some $\beta$-dependent value: investing extra compute in the approximation of the fixed point does not pay off if $\beta$ is not decreased accordingly.

\subsection{Proof of Theorem~\ref{thm:bound_ep_estimate}}

As mentioned above, Theorem~\ref{thm:bound_ep_estimate} can be proved by individually bounding the two kind of errors that compose the meta-gradient estimation error, that are the finite difference error and the solution approximation induced error.

The $B^\mathrm{learn}(\delta+\delta')/\beta + B^\mathrm{eval}\delta'$ part of the bound stems from the solution approximation error, and can be obtained by using the assumption that the partial derivatives of $\Li$ and $\Lo$ are Lipschitz continuous.

Bounding the finite difference error requires more work. We use Taylor's theorem to show that $\widehat{\nabla}_\theta^*-\nabla_\theta$ is equal to some integral remainder. It then remains to bound what is inside the integral remainder, which is the second order derivative $\mathrm{d}_\beta^2\partial_\theta\Lt(\phisb,\beta)$. This is done in the Lemmas presented in this section: Lemma~\ref{lem:bound_phisb} allows us to get uniform bounds, Lemmas~\ref{lem:norm_der_phisb} and \ref{lem:norm_der2_phisb} bound the first and second order derivatives of $\beta\mapsto\phisb$ and Lemma~\ref{lem:pderbtheta_to_phisb} bounds $\mathrm{d}_\beta^2\partial_\theta\Lt(\phisb,\beta)$ with the norm of the two derivatives we have just bounded. We present the proofs for those four lemmas in Section~\ref{subsec:proof_technical_lemmas}.

\begin{lemma}
    \label{lem:bound_phisb}
    Under Assumption~\ref{ass_app:regularity_inner}.$ii$, if $\theta$ lies in a compact set $\mathcal{D}$ the function $(\theta, \beta) \mapsto \phi^*_{\theta, \beta}$ is uniformly bounded.
\end{lemma}

\begin{lemma}
	\label{lem:norm_der_phisb}
    Under Assumption~\ref{ass_app:regularity_inner}.$ii$, there exists a $\theta$-dependent constant $R$ s.t., for every positive $\beta$,
    \begin{equation*}
        \norm{\der{\beta}{\phisb}} \leq \frac{LR}{(1+\beta)^2\mu}.
    \end{equation*}
    If we additionally assume that $\theta$ lies in a compact set, we can choose $R$ to be independent of $\theta$.
\end{lemma}

\begin{remark}
    \label{rmk:dist_phisb}
    A side product of the proof of Lemma~\ref{lem:norm_der_phisb} is a bound on the distance between the minimizer of $\Lt$ and the minimizers of $\Li$ and $\Lo$. We have  
    \begin{equation*}
        \snorm{\phisb-\phisi} \leq \frac{1}{1+\beta}
    \end{equation*}
    and
    \begin{equation*}
        \snorm{\phisb-\phisz} \leq \frac{\beta}{1+\beta}
    \end{equation*}
    up to some constant factors.
\end{remark}

\begin{lemma}
    \label{lem:norm_der2_phisb}
    Under Assumptions~\ref{ass_app:regularity_inner}.$ii$ and \ref{ass_app:regularity_inner}.$iii$,
    \begin{equation*}
        \norm{\der{\beta^2}{^2\phisb}} \leq \frac{\rho}{\mu}\norm{\der{\beta}{\phisb}}^2 + \frac{2L}{(1+\beta)\mu}\norm{\der{\beta}{\phisb}}\!.
    \end{equation*}
\end{lemma}

When Lemma~\ref{lem:norm_der2_phisb} is combined with Lemma \ref{lem:norm_der_phisb},
\begin{equation}
    \norm{\der{\beta^2}{^2\phisb}} \leq \frac{1}{(1+\beta)^{3}}.
\end{equation}
up to some constant factor.

\begin{lemma}
    \label{lem:pderbtheta_to_phisb}
    Under Assumptions~\ref{ass_app:regularity_inner}.$ii$, \ref{ass_app:regularity_inner}.$iii$ and \ref{ass_app:regularity_inner}.$iv$, there exists a constant $M$ such that
    \begin{equation*}
        \norm{\der{\beta^2}{^2}\pder{\theta}{\Lt}(\phisb,\beta)} \leq M \left ( \norm{\der{\beta}{\phisb}} + \left (1+\beta \right)\left( \norm{\der{\beta}{\phisb}}^2 + \norm{\der{\beta^2}{^2\phisb}}\right) \right ) \!.
    \end{equation*}
\end{lemma}

We can now prove Theorem~\ref{thm:bound_ep_estimate} using the four lemmas that we have just presented. Note that we omit the $\theta$-dependency whenever $\theta$ is fixed, for the sake of conciseness.
\begin{proof}[Proof of Theorem~\ref{thm:bound_ep_estimate}]
    We separate the sources of error within the meta-gradient estimation error using the triangle inequality:
    \begin{equation}
        \snorm{\widehat{\nabla}_\theta - \nabla_\theta} \leq \underbrace{\snorm{ \widehat{\nabla}_\theta - \widehat{\nabla}_\theta^*}}_{a)} + \underbrace{\snorm{ \widehat{\nabla}^*_\theta - \nabla_\theta}}_{b)},
    \end{equation}
    and bound the two terms separately:
    \begin{itemize}
        \item[a)] Recall that 
        \begin{equation}
            \widehat{\nabla}_\theta = \frac{1}{\beta}\left ( \pder{\theta}{\Lt}(\phihb, \beta)-\pder{\theta}{\Lt}(\phihz, 0)\right )
        \end{equation}
        and that a similar formula holds for $\widehat{\nabla}_\theta^*$ (evaluated at the true solutions instead of the approximations). It follows 
        \begin{equation}
            \snorm{ \widehat{\nabla}_\theta - \widehat{\nabla}_\theta^*} \leq \frac{1}{\beta} \left (\norm{\pder{\theta}{\Lt}(\phihb, \beta ) - \pder{\theta}{\Lt}\left(\phisb, \beta\right)}\right. + \left .\norm{\pder{\theta}{\Lt}(\phihz, 0) - \pder{\theta}{\Lt}(\phisz, 0)} \right)\!.
        \end{equation}
        Since $\phi\mapsto\partial_\theta\Lt(\phi,\beta)$ is a $(B^{\mathrm{learn}}+\beta B^{\mathrm{eval}})$-Lipschitz function as a sum of $\partial_\theta\Li$ and $\partial_\theta\Lo$, two Lipschitz continuous functions with constants $B^{\mathrm{learn}}$ and $B^{\mathrm{eval}}$,
        \begin{align}
            \snorm{ \widehat{\nabla}_\theta - \widehat{\nabla}_\theta^*} & \leq \frac{B^{\mathrm{learn}}+\beta B^{\mathrm{eval}}}{\beta}\snorm{\phihb-\phisb} + \frac{B^{\mathrm{learn}}}{\beta}\snorm{\phihz-\phisz}\\
            & \leq \frac{B^{\mathrm{learn}}+\beta B^{\mathrm{eval}}}{\beta}\delta' + \frac{B^{\mathrm{learn}}}{\beta}\delta.
        \end{align}
        \item[b)] Taylor's theorem applied to $\beta \mapsto \partial_\theta{\Lt}(\phisb, \beta)$ up to the first order of differentiation yields
        \begin{equation}
            \pder{\theta}{\Lt}\left(\phisb, \beta \right) = \pder{\theta}{\Lt}\left(\phisz, 0\right) + \beta  \der{\beta}{}\pder{\theta}{\Lt}\left(\phisz, 0\right) + \int_{0}^{\beta}(\beta-t) \der{\beta^2}{^2}\pder{\theta}{\Lt}\left(\phi_t^*, t\right)\mathrm{d}t.
        \end{equation}
        The equilibrium propagation theorem (Theorem \ref{thm:eq_prop}), which is applicable thanks to Assumption~\ref{ass_app:regularity_inner}.$ii$, gives
        \begin{equation}
            \nabla_\theta = \der{\beta}{}\pder{\theta}{\Lt}(\phisz, 0),
        \end{equation}
        hence 
        \begin{equation}
            \snorm{\widehat{\nabla}_\theta^* - \nabla_\theta} = \norm{\int_{0}^{\beta}(\beta-t) \der{\beta^2}{^2}\pder{\theta}{\Lt}\left(\phi_t^*, t\right)\mathrm{d}t}\!.
        \end{equation}
        Using the integral version of the Cauchy-Schwartz inequality, we have
        \begin{equation}
            \snorm{\widehat{\nabla}_\theta^* - \nabla_\theta} \leq \int_{0}^{\beta}(\beta-t)\norm{\der{\beta^2}{^2} \pder{\theta}{\Lt}(\phi_t^*, t)} \mathrm{d}t.
        \end{equation}
        We now use Lemma~\ref{lem:pderbtheta_to_phisb} combined with Lemmas~\ref{lem:norm_der_phisb} and \ref{lem:norm_der2_phisb} to bound $\mathrm{d}_\beta^2\partial_\theta\Lt(\phi^*_{t},t)$. We focus on the $\beta$ dependencies and omitting constant factors:
        \begin{align*}
                \norm{\der{\beta^2}{^2}\pder{\theta}{\Lt}(\phi^*_t,t)} &\leq \norm{\der{\beta}{\phi^*_t}} + (1+t)\left(\norm{\der{\beta}{\phi^*_t}}^2 + \norm{\der{\beta^2}{^2\phi^*_t}}\right) \\
                & \leq \frac{1}{(1+t)^2} + (1+t)\left(\frac{1}{(1+t)^3} + \frac{1}{(1+t)^4}\right)\\
                & \leq (1+t)^{-2}.
        \end{align*}
        It follows that
        \begin{equation}
            \begin{split}
                \snorm{\widehat{\nabla}_\theta^* - \nabla_\theta} & \leq \int_0^\beta \frac{(\beta-t)}{(1+t)^2}\mathrm{d}t\\
                &= (1+\beta)\int_0^\beta \frac{1}{(1+t)^2}\mathrm{d}t - \int_0^\beta \frac{1}{(1+t)}\mathrm{d}t\\
                & =(1+\beta)\frac{\beta}{1+\beta} - \ln(1+\beta)\\
                & \leq \beta - \frac{\beta}{1+\beta}\\
                & = \frac{\beta^2}{1+\beta}.
            \end{split}
        \end{equation}
        where the inequality comes from the well-known $\ln(x)\geq 1 - \frac{1}{x}$ inequality for positive $x$ (applied to $x=1+\beta$). There hence exists a constant $C$ such that
        \begin{equation}
            \snorm{\widehat{\nabla}_\theta^* - \nabla_\theta} \leq C\frac{\beta}{1+\beta}.
        \end{equation}
    \end{itemize}
    If $\theta$ lies in a compact set, the bound in Lemma~\ref{lem:norm_der_phisb} is uniform over $\theta$. This is the only constant factor that depends on $\theta$, so the bound is uniform.
\end{proof}

\subsection{Proof of technical lemmas}
\label{subsec:proof_technical_lemmas}

In this section, we prove the four technical lemmas that we need for Theorem~\ref{thm:bound_ep_estimate}.

\paragraph{Proof of Lemma~\ref{lem:bound_phisb}}
\begin{replemma}{lem:bound_phisb}
    Under Assumption~\ref{ass_app:regularity_inner}.$ii$, if $\theta$ lies in a compact set $\mathcal{D}$ the function $(\theta, \beta) \mapsto \phi^*_{\theta, \beta}$ is uniformly bounded.
\end{replemma}
\begin{proof}
    Let $\alpha \in [0, 1]$. Define
    \begin{equation}
        \Lt'(\phi, \theta, \alpha) := (1-\alpha) \Li(\phi, \theta) + \alpha\Lo(\phi, \theta).
    \end{equation}
    As $\Li$ and $\Lo$ are strongly-convex, there exists a unique minimizer $\phi^{*\prime}_{\theta,\alpha}$ of $\phi\mapsto\Lt'(\phi,\theta, \alpha)$. The implicit function theorem ensures that the function $(\theta, \alpha) \mapsto \phi^{*\prime}_{\theta,\alpha}$, defined on $\mathcal{D} \times [0, 1]$, is continuous. As $\mathcal{D} \times [0, 1]$ is a compact set, $\phi^{*\prime}_{\theta,\alpha}$ is then uniformly bounded. Now, remark that
    \begin{equation}
        \Lt(\phi, \theta, \beta) = (1+\beta)\Lt'\left(\phi, \theta, \frac{\beta}{1+\beta}\right)
    \end{equation}
    and thus $\phisbt = \phi^{*\prime}_{\theta, \beta/(1+\beta)}$. It follows that $\phisbt$ is uniformly bounded.
\end{proof}

\paragraph{Proof of Lemma~\ref{lem:norm_der_phisb}}
\begin{replemma}{lem:norm_der_phisb}
    Under Assumption~\ref{ass_app:regularity_inner}.$ii$, there exists a $\theta$-dependent constant $R$ s.t., for every positive $\beta$,
    \begin{equation*}
        \norm{\der{\beta}{\phisb}} \leq \frac{LR}{(1+\beta)^2\mu}.
    \end{equation*}
    If we additionally assume that $\theta$ lies in a compact set, we can choose $R$ to be independent of $\theta$.
\end{replemma}
\begin{proof}
    The function $\phi \mapsto \Lt(\phi,\beta)$ is $(1+\beta)\mu$-strongly convex so its Hessian $\partial_\phi^2\Lt$ is invertible and its inverse has a spectral norm upper bounded by $1/((1+\beta)\mu)$. The use of the implicit function theorem follows and gives
    \begin{equation}
        \begin{split}
             \norm{\sder{\beta}{\phisb}} & = \snorm{-\left( \partial_\phi^2\Lt(\phisb,\beta)\right)^{-1} \partial_\beta\partial_\phi\Lt(\phisb)}\\
             & = \snorm{-\left( \partial_\phi^2\Lt(\phisb,\beta)\right)^{-1} \partial_\phi\Lo(\phisb)}\\
             & \leq \frac{1}{(1+\beta)\mu}\snorm{\partial_\phi\Lo(\phisb)}.
         \end{split}
    \end{equation}
    It remains to bound the gradient of $\Lo$. Since $\beta \mapsto \phisb$ is continuous and has finite limits in 0 and $\infty$ (namely the minimizers of $\Li$ and $\Lo$), it evolves in a bounded set. There hence exists a positive constant $R$ such that, for all positive $\beta$,
    \begin{equation}
        \max \left ( \norm{\phisb - \phisz}, \norm{\phisb - \phisi} \right ) \leq \frac{R}{2}.
    \end{equation}
        If $\theta$ lies in a compact set, Lemma~\ref{lem:bound_phisb} guarantees that there exists such a constant that doesn't depend on the choice of $\theta$.
    We then bound the gradient of $\Lo$ using the smoothness properties of $\Li$ and $\Lo$, either directly
    \begin{equation}
        \snorm{\partial_\phi \Lo(\phisb)} \leq L\snorm{\phisb-\phisi} \leq \frac{LR}{2}
    \end{equation}
    or indirectly, using the fixed point condition $\partial_\phi\Lt(\phisb,\beta)=0$,
    \begin{equation}
        \snorm{\partial_\phi \Lo(\phisb)} =  \frac{1}{\beta}\snorm{-\partial_\phi \Li(\phisb)} \leq \frac{L\snorm{\phisb - \phisz}}{\beta}  \leq \frac{LR}{2\beta}.
    \end{equation}
    The required result is finally obtained by remarking
    \begin{equation}
        \snorm{\partial_\phi \Lo(\phisb)} \leq \min\left(1,\frac{1}{\beta}\right)\frac{LR}{2}\leq \frac{LR}{1+\beta}.
    \end{equation}
\end{proof}

\paragraph{Proof of Remark~\ref{rmk:dist_phisb}}
\label{subsubsec:proof_remark_dist_phisb}
We now prove Remark~\ref{rmk:dist_phisb}, which directly follows from the previous proof. Recall that we have just proved
\begin{equation}
    \snorm{\partial_\phi \Lo(\phisb)} \leq \frac{LR}{1+\beta}.
\end{equation}
With the strong convexity of $\Lo$, the gradient is also lower bounded
\begin{equation}
    \snorm{\partial_\phi \Lo(\phisb)} \geq \mu \snorm{\phisb-\phisi},
\end{equation}
meaning that
\begin{equation}
    \snorm{\phisb - \phisi} \leq \frac{LR}{\mu(1+\beta)}.
\end{equation}
Similarly, one can show that
\begin{equation}
    \snorm{\phisz-\phisb} \leq \frac{\beta}{1+\beta}
\end{equation}
up to some constant factor.
This can be proved with
\begin{equation}
    \snorm{\phisz-\phisb} \leq \frac{\snorm{\partial_\phi \Li(\phisb)}}{\mu} = \frac{\beta\snorm{\partial_\phi \Lo(\phisb)}}{\mu} \leq \frac{\beta LR}{(1+\beta)\mu}.
\end{equation}

\paragraph{Proof of Lemma~\ref{lem:norm_der2_phisb}}
\begin{replemma}{lem:norm_der2_phisb}
    Under Assumptions~\ref{ass_app:regularity_inner}.$ii$ and \ref{ass_app:regularity_inner}.$iii$,
    \begin{equation*}
        \norm{\der{\beta^2}{^2\phisb}} \leq \frac{\rho}{\mu}\norm{\der{\beta}{\phisb}}^2 + \frac{2L}{(1+\beta)\mu}\norm{\der{\beta}{\phisb}}\!.
    \end{equation*}
\end{replemma}
\begin{proof}
    The starting point of the proof is the implicit function theorem, that we differentiate with respect to $\beta$ as a product of functions
    \begin{equation}
        \begin{split}
            \der{\beta^2}{^2\phisb} &= \der{\beta}{} \left ( -\left( \pder{\phi^2}{^2\Lt}(\phisb,\beta)\right)^{-1} \pder{\phi}{\Lo}(\phisb) \right )\\
            &= -\underbrace{\left(\der{\beta}{} \pder{\phi^2}{^2\Lt}(\phisb,\beta) ^{-1}\right)\pder{\phi}{\Lo}(\phisb)}_{a)} - \underbrace{\pder{\phi^2}{^2\Lt}(\phisb,\beta)^{-1} \left(\der{\beta}{} \pder{\phi}{\Lo}(\phisb)\right)}_{b)}\!.
        \end{split}
    \end{equation}
    We now individually calculate and bound each term.
    \begin{itemize}
        \item[a)] The differentiation of the inverse of a matrix gives
        \begin{equation}
            a) = -\partial_\phi^2\Lt(\phisb,\beta)^{-1}\left(\sder{\beta}{}\partial_\phi^2\Lt(\phisb,\beta)\right) \partial_\phi^2\Lt(\phisb,\beta)^{-1}\partial_\phi\Lo(\phisb),
        \end{equation}
        which we can rewrite as
        \begin{equation}
            a) = \partial_\phi^2\Lt(\phisb,\beta)^{-1}\left(\sder{\beta}{}\partial_\phi^2\Lt(\phisb,\beta)\right)\sder{\beta}{\phisb}.
        \end{equation}
        The derivative term in the middle of the right hand side is equal to 
        \begin{equation}
            \begin{split}
                \sder{\beta}{}\partial_\phi^2\Lt(\phisb,\beta) & = \sder{\beta}{} \left [ \partial_\phi^2 \Li(\phisb) + \beta \partial_\phi^2\Lo(\phisb) \right ] \\
                &=\sder{\beta}{}\partial_\phi^2\Li(\phisb) + \beta\sder{\beta}{}\partial_\phi^2\Lo(\phisb) + \partial_\phi^2\Lo(\phisb).
            \end{split}
        \end{equation}
        Using the Lipschitz continuity of the Hessians, 
        \begin{equation}
            \norm{\sder{\beta}{}\partial_\phi^2\Li(\phisb) + \beta\sder{\beta}{}\partial_\phi^2\Lo(\phisb)} \leq (1+\beta)\rho\norm{\sder{\beta}{\phisb}}.
        \end{equation}
        We can upper bound the norm of the Hessian of $\Lo$ by $L$ as $\Lo$ is L-smooth. The last two equations hence give
        \begin{equation}
            \norm{\sder{\beta}{}\partial_\phi^2\Lt(\phisb,\beta)} \leq (1+\beta)\rho\norm{\sder{\beta}{\phisb}}+L.
        \end{equation}
        We finally have
        \begin{equation}
            \begin{split}
                \norm{a)} & \leq \frac{1}{\mu(1+\beta)}\left((1+\beta)\rho\norm{\sder{\beta}{\phisb}}+L\right) \norm{\sder{\beta}{\phisb}}\\
                & \leq \frac{\rho}{\mu}\norm{\sder{\beta}{\phisb}}^2 + \frac{L}{(1+\beta)\mu}\norm{\sder{\beta}{\phisb}}.
            \end{split}
        \end{equation}
        \item[b)] With the chain rule,
        \begin{equation}
            \sder{\beta}{}\partial_\phi \Lo(\phisb) = \partial^2_\phi\Lo(\phisb)\sder{\beta}{\phisb} 
        \end{equation}
        so
        \begin{equation}
            \begin{split}
                \norm{b)} & \leq \norm{{\partial_\phi^2\Lt(\phisb,\beta)^{-1}}} \norm{\partial^2_\phi\Lo(\phisb)} \norm{\sder{\beta}{\phisb}}\\
                & \leq \frac{L}{(1+\beta)\mu} \norm{\sder{\beta}{\phisb}}.
            \end{split}
        \end{equation}
    \end{itemize}
\end{proof}

\paragraph{Proof of Lemma~\ref{lem:pderbtheta_to_phisb}}
\begin{replemma}{lem:pderbtheta_to_phisb}
    Under Assumptions~\ref{ass_app:regularity_inner}.$ii$, \ref{ass_app:regularity_inner}.$iii$ and \ref{ass_app:regularity_inner}.$iv$, there exists a constant $M$ such that
    \begin{equation*}
        \norm{\der{\beta^2}{^2}\pder{\theta}{\Lt}(\phisb,\beta)} \leq M \left ( \norm{\der{\beta}{\phisb}} + \left (1+\beta \right)\left( \norm{\der{\beta}{\phisb}}^2 + \norm{\der{\beta^2}{^2\phisb}}\right) \right ) \!.
    \end{equation*}
\end{replemma}
\begin{proof}
    We want to bound the norm of $\mathrm{d}^2_\beta \partial_\theta\Lt(\phisb,\beta)$. The first order derivative can be calculated with the chain rule of differentiation
    \begin{equation}
        \sder{\beta}{}\spder{\theta}{\Lt}(\phisb,\beta) = \partial_\beta\partial_\theta\Lt(\phisb,\beta)+\partial_\phi\partial_\theta\Lt(\phisb,\beta)\sder{\beta}{\phisb}.
    \end{equation}
    We then once again differentiate this equation with respect to $\beta$. The $\partial_\beta \partial_\theta\Lt(\phisb,\beta)$ term has in fact, due to the nature of $\Lt$, no direct dependence on $\beta$ and is equal to $\partial_\theta \Lo(\phisb)$. Hence
    \begin{equation}
        \sder{\beta}{}\partial_\beta\partial_\theta\Lt(\phisb,\beta) = \partial_\phi\partial_\theta\Lo(\phisb)\sder{\beta}{\phisb}.
    \end{equation}
    Differentiating the other term yields
    \begin{multline}
        \sder{\beta}{}\left [ \partial_\phi\partial_\theta\Lt(\phisb,\beta)\sder{\beta}{\phisb} \right ] =  \left [ \partial_\beta\partial_\phi\partial_\theta\Lt(\phisb,\beta)+\partial_\phi^2\partial_\theta\Lt(\phisb,\beta) \otimes \sder{\beta}{\phisb} \right ] \sder{\beta}{\phisb} +\\ \partial_\phi\partial_\theta\Lt(\phisb,\beta)\sder{\beta^2}{\phisb}.
    \end{multline}
    Therefore,
    \begin{multline}
        \sder{\beta^2}{}\partial_\theta\Lt(\phisb,\beta) = 2\partial_\phi\partial_\theta\Lo(\phisb)\sder{\beta}{\phisb} + \partial^2_\phi\partial_\theta\Lt(\phisb,\beta) \otimes \sder{\beta}{\phisb} \otimes \sder{\beta}{\phisb} +\\ \partial_\phi\partial_\theta\Lt(\phisb,\beta)\sder{\beta^2}{\phisb}.
    \end{multline}
    We now individually bound each term:
    \begin{itemize}
        \item[--] due to Assumption~\ref{ass_app:regularity_inner}.$i$, $\phi\mapsto\partial_\theta\Lo(\phi)$ is $B^\mathrm{eval}$-Lipschitz continuous, so $\snorm{\partial_\phi\partial_\theta\Lo} \leq B^\mathrm{eval}$ and
        \begin{equation}
            \norm{2\partial_\phi\partial_\theta\Lo(\phisb)\sder{\beta}{\phisb}} \leq 2B^\mathrm{eval} \norm{\sder{\beta}{\phisb}}\!.
        \end{equation}
        \item[--] similarly to the previous point, 
        \begin{equation}
            \norm{\partial_\phi\partial_\theta\Lt(\phisb)\sder{\beta^2}{\phisb}} \leq (B^\mathrm{learn}+\beta B^\mathrm{eval}) \norm{\sder{\beta^2}{\phisb}}\!.
        \end{equation}
        \item[--] Assumption~\ref{ass_app:regularity_inner}.$iv$ ensures that $\phi\mapsto\partial_\phi\partial_\theta\Lt(\phi,\beta)$ is $(1+\beta)\sigma$-Lipschitz continous and 
        \begin{equation}
            \norm{\partial_\phi^2\partial_\theta \Lt(\phisb,\beta) \otimes \sder{\beta}{\phisb} \otimes \sder{\beta}{\phisb}} \leq (1+\beta)\sigma\norm{\sder{\beta}{\phisb}}^2\!.
        \end{equation}
    \end{itemize}
    Take $M := \max(2B^\mathrm{eval}, B^\mathrm{learn}, \sigma)$: we now have the desired result.
\end{proof}

\subsection{A corollary of Theorem~\ref{thm:bound_ep_estimate}}

Theorem~\ref{thm:bound_ep_estimate} highlights the importance of considering $\beta$ as a hyperparameter of the learning rule that needs to be adjusted to yield the best possible meta-gradient estimate. Corollary~\ref{cor_app:opt_bound} removes the dependence in $\beta$ and considers the best achievable bound under given fixed point approximation errors.
\begin{corollary}
    \label{cor_app:opt_bound}
    Under Assumption~\ref{ass_app:regularity_inner}, if we suppose that for every strictly positive $\beta$ we approximate the two fixed points with precision $\delta$ and $\delta'$ and if $(\delta+\delta')< C/B^\mathrm{learn}$, the best achievable bound in Theorem~\ref{thm:bound_ep_estimate} is smaller than
    \begin{equation*}
        B^\mathrm{eval}\delta'+2\sqrt{CB^\mathrm{learn}(\delta+\delta')}
    \end{equation*}
    and is attained for $\beta$ equal to
    \begin{equation*}
        \beta^*(\delta, \delta') = \frac{\sqrt{B^\mathrm{learn}(\delta+\delta')}}{\sqrt{C}-\sqrt{B^\mathrm{learn}(\delta+\delta')}}.
    \end{equation*}
\end{corollary}

\begin{figure}[htbp]
    \centering
    \includegraphics[width=4.25in]{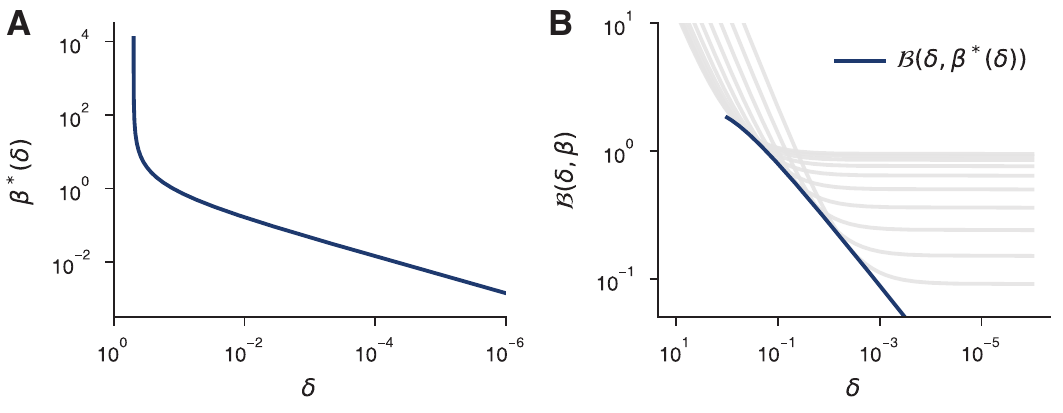}
    \caption{Visualization of Corollary \ref{cor_app:opt_bound}. (A) $\beta$ value that minimizes the bound, as a function of $\delta=\delta'$. (B) Best achievable bound as a function of $\delta=\delta'$ in blue (more precisely the one before the last upper bound in the proof). The grey lines are the bounds from Theorem~\ref{thm:bound_ep_estimate} we laid out on Fig.~\ref{fig_app:bound_viz}C.} 
    \label{fig_app:viz_cor_opt_bound}
\end{figure}

The most limiting part of the bound depends on the sum $\delta + \delta'$ and not on the individual quantities, suggesting that the two errors should be of the same magnitude to avoid unnecessary computations.

\begin{proof}
    The $\beta$ derivative of the bound $\mathcal{B}$ obtained in Theorem~\ref{thm:bound_ep_estimate} is
    \begin{equation}
        \pder{\beta}{\mathcal{B}}(\delta, \delta', \beta)=-\frac{B^\mathrm{learn}(\delta+\delta')}{\beta^{2}} + \frac{C}{(1+\beta)^2}
    \end{equation}
    and vanishes for $\beta$ verifying
    \begin{equation}
        \beta\left(\sqrt{C}-\sqrt{B^\mathrm{learn}(\delta+\delta')}\right) = \sqrt{B^\mathrm{learn}(\delta+\delta')}.
    \end{equation}
    As $(\delta+\delta')< C/B^\mathrm{learn}$, the previous criterion is met when $\beta$ is equal to the positive
    \begin{equation}
        \beta^* := \frac{\sqrt{B^\mathrm{learn}(\delta+\delta')}}{\sqrt{C}-\sqrt{B^\mathrm{learn}(\delta+\delta')}}.
    \end{equation}
    The optimal bound is then
    \begin{align*}
        \mathcal{B}(\delta, \delta', \beta^*) &= B^\mathrm{eval}\delta'+ \sqrt{B^\mathrm{learn}(\delta+\delta')}\left ( \sqrt{C} - \sqrt{B^\mathrm{learn}(\delta+\delta')} \right)+ \sqrt{CB^\mathrm{learn}(\delta+\delta')}\\
        & \leq B^\mathrm{eval} \delta' + 2\sqrt{CB^\mathrm{learn}(\delta + \delta')}.
    \end{align*}
\end{proof}

\subsection{Verification of the theoretical results on an analytical problem}

We investigate a quadratic approximation of the complex synapse model, in which everything can be calculated in closed form and where the assumptions needed for the theory hold. Define $\Li$ and $\Lo$ as follows\footnote{In our experiments, we take the dimension of the parameter space $N$ to be equal to 50. The Hessian is taken to be $\mathrm{diag}(1, ..., 1/N)$. $\omega$ is randomly generated according to $\omega \sim \mathcal{N}(0,\sigma_\omega)$ with $\sigma_\omega = 2$. $\phi^l$ and $\phi^e$ are drawn around $\phi^{\tau}\sim\mathcal{N}(0, \sigma_\tau)$ (with $\sigma_\tau=1$).}:
\begin{equation*}
    \begin{split}
        \Li(\phi,\omega) &= \frac{1}{2}(\phi-\phi^l)^\top H (\phi-\phi^{l}) + \frac{\lambda}{2}\snorm{\phi-\omega}^2\\
        \Lo(\phi) &= \frac{1}{2}(\phi-\phi^{e})^\top H (\phi-\phi^{e})
    \end{split}
\end{equation*}
where $\lambda$ is a scalar that controls the strength of the regularization that we consider fixed, $\phi^l$ and $\phi^e$ two vectors and $H$ a positive definite diagonal matrix. The rationale behind this approximation is the following: the data-driven learning and evaluation losses share the same curvature but have different minimizers, respectively $\phi^l$ and $\phi^e$. The matrix $H$ then models the Hessian and we consider it diagonal for simplicity. Thanks to the quadratic approximation, many quantities involved in our contrastive meta-learning rule can be calculated in closed form.

\paragraph{Calculation of the finite difference error.}

A formula for the minimizer of $\Lt=\Li+\beta\Lo$ can be derived analytically. The derivative of $\Lt$ vanishes if and only if
\begin{equation*}
    \left((1+\beta)H +  \lambda\Id \right)\phi - H\phi^l - \beta H\phi^e - \lambda\omega = 0,
\end{equation*}
hence
\begin{equation*}
    \phisbt =  \left((1+\beta)\Id +  \lambda H^{-1} \right)^{-1} \left( \phi^l + \beta \phi^e + \lambda H^{-1}\omega\right)\!.
\end{equation*}
$\lambda H^{-1}$ is an interesting quantity in this example. It acts as the effective per-coordinate regularization strength: regularization will be stronger on flat directions. 

The meta-gradient calculation follows. As
\begin{equation*}
    \partial_\omega \Lt(\phi, \omega, \beta) = -\lambda(\phi-\omega),
\end{equation*}
the use of the equilibrium propagation theorem (Theorem~\ref{thm:eq_prop}) gives
\begin{equation*}
    \begin{split}
        \nabla_\omega &= \left . \der{\beta}{}\pder{\omega}{\Lt}(\phisbt, \omega, \beta) \right |_{\beta=0}\\
        &= \pder{\phi\partial\omega}{^2\Li}(\phiszt, \omega)\evalat{\der{\beta}{\phisbt}}{\beta=0} + 0\\
        &= -\lambda\evalat{\der{\beta}{\phisbt}}{\beta=0}\!.
    \end{split}
\end{equation*}
It now remains to calculate the derivative of $\phisbt$ with respect to $\beta$ using the formula of $\phisbt$:
\begin{equation*}
    \begin{split}
        \der{\beta}{\phisbt} = &  \left((1+\beta)\Id+\lambda H^{-1}\right)^{-1}\phi^e \\
        & -\left((1+\beta)\Id+\lambda H^{-1}\right)^{-2}(\phi^l + \beta \phi^e + \lambda H^{-1}\omega)\\
        = & \left((1+\beta)\Id+\lambda H^{-1}\right)^{-2}\left( (1+\beta)\phi^e+\lambda H^{-1}\phi^e -\phi^l -\beta\phi^e-\lambda H^{-1}\omega \right)\\
        = & \left((1+\beta)\Id+\lambda H^{-1}\right)^{-2} \left ( (\phi^e-\phi^l) + \lambda H^{-1} ( \phi^e - \omega ) \right )
    \end{split}
\end{equation*}
Define $\psi := (\phi^e-\phi^l) + \lambda H^{-1} ( \phi^e - \omega )$; the meta-gradient finally is
\begin{equation*}
    \nabla_\omega = -\lambda(\Id+\lambda H^{-1})^{-2}\psi.
\end{equation*}

We can now calculate the finite difference error. Recall the equilibrium propagation estimate at fixed points
\begin{equation*}
    \widehat{\nabla}_\omega^* = \frac{1}{\beta}\left(\pder{\omega}{\Lt}(\phisbt,\omega, \beta) - \pder{\omega}{\Lt}(\phiszt,\omega, 0) \right)\!.
\end{equation*}
In this formulation, it is equal to
\begin{align*}
    \widehat{\nabla}_\omega^* &= -\frac{\lambda}{\beta}(\phisbt-\phiszt)\\
    &= -\lambda\left((\Id +\lambda H^{-1})((1+\beta)\Id+\lambda H^{-1})\right)^{-1}\psi\\
    &= (\Id + \lambda H^{-1})\left((1+\beta)\Id + \lambda H^{-1})\right)^{-1} \nabla_\omega.
\end{align*}

The finite difference can now be lower and upper bounded. First,
\begin{equation*}
    \nabla_\omega - \widehat{\nabla}_\omega^* = \beta\left((1+\beta)\Id+\lambda H^{-1}\right)^{-1}\nabla_\omega.
\end{equation*}
Introduce $\mu$ the smallest eigenvalue of $H$ and $L$ its largest one. We then have
\begin{equation}
    \frac{\mu\beta}{(1+\beta)\mu+\lambda} \norm{\nabla_\omega}\leq \snorm{\nabla_\omega - \widehat{\nabla}_\omega^*} \leq \frac{L\beta}{(1+\beta)L+\lambda}\snorm{\nabla_\omega}.
\end{equation}
This shows that the finite difference error part of Theorem \ref{thm:bound_ep_estimate} is tight and, in this case, accurately describes the behavior of the finite difference error as a function of $\beta$.

\paragraph{Empirical results.}

\begin{figure}[htbp]
    \centering
    \includegraphics[width=5.5in]{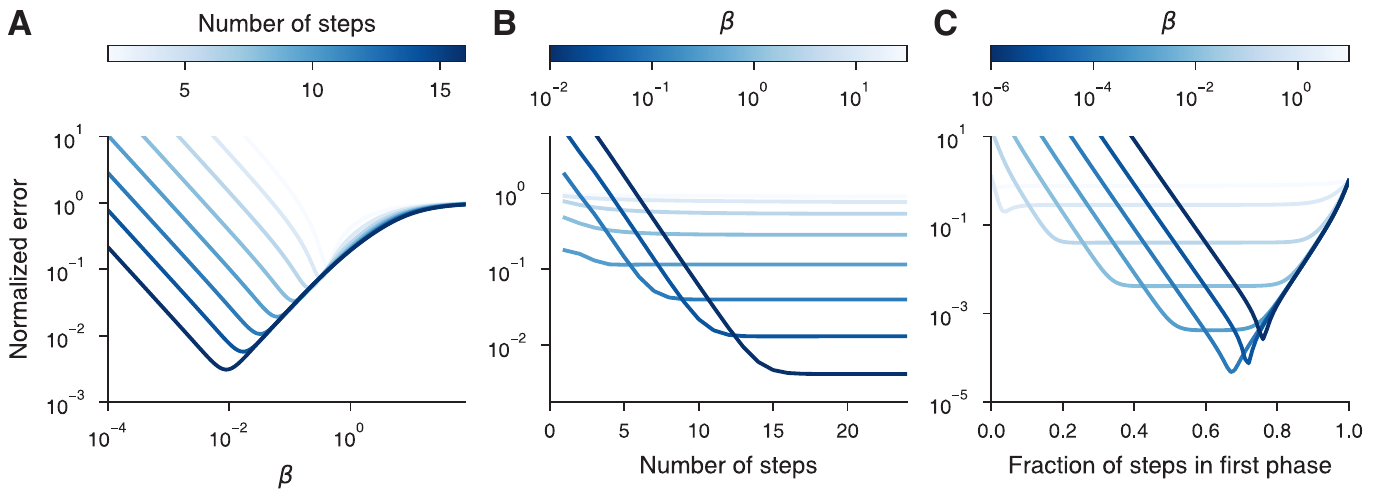}
    \caption{Empirical verification of the theoretical results on an analytical quadratic approximation of the synaptic model. We plot the normalized error between the meta-gradient estimate $\widehat{\nabla}_\theta$ and the true one $\nabla_\theta$, as a function of $\beta$ (A), of the number of steps in the two phases (which is a proxy for $-\log\delta$ and $-\log\delta'$ used in the theory) (B), and as a function of the allocation of the computational resources between the two phases, the total number of steps being fixed to 100 (C).}
    \label{fig_app:viz_quadratic_boundbeta}
\end{figure}

The solution approximation induced error part of the bound cannot be treated analytically as it depends on $\delta$ and $\delta'$, which are in essence empirical quantities. We cannot directly control them either. Instead, we use the number of gradient descent steps to minimize $\Lt$ as a proxy, that is closely related to $-\log \delta$ when gradient descent has a linear convergence rate. We choose the number of steps to be the same in the two phases, for the sake of simplicity, even though it may not be optimal. 
We plot the evolution of the normalized error
\begin{equation*}
    \frac{\snorm{\nabla_\omega- \widehat{\nabla}_\omega}}{\snorm{\nabla_\omega}}
\end{equation*}
between the meta-gradient and the contrastive estimate \eqref{eqn_app:contrastive_update} in Fig.~\ref{fig_app:viz_quadratic_boundbeta}. The qualitative behavior of this error, as a function of $\beta$ (Fig.~\ref{fig_app:viz_quadratic_boundbeta}A) and of number of steps (Fig.~\ref{fig_app:viz_quadratic_boundbeta}B), is accurately captured by Theorem~\ref{thm:bound_ep_estimate} (compare with Fig.~\ref{fig_app:bound_viz}A and C).

We finish the study of this quadratic model by probing the $(\delta, \delta')$ space in a different way, by fixing the total number of steps and then modifying the allocation across the two phases (Fig.~\ref{fig_app:viz_quadratic_boundbeta}C). The best achievable error, as a function of $\beta$, decreases before some $\beta^*$ value and then increases, following the predictions of Theorem~\ref{thm:bound_ep_estimate}: too small $\beta$ values turn out to hurt performance when the solutions cannot be approximated arbitrarily well. Interestingly, the error plateaus for large $\beta$ values and the size of the plateau decreases with $\beta$ until reaching a critical value where it disappears. A conservative choice in practice is therefore to overestimate $\beta$, as it reduces the meta-gradient estimation sensitivity to a sub-optimal allocation, with only a minor degradation in the best achievable quality.

\subsection{Verification the theoretical results on a simple hyperparameter optimization task}

We now move to a more complicated setting that is closer to problems of practical interest, and in which we are not guaranteed that the assumptions of the theory hold. Still, it is simple enough such that we can calculate the exact value of the meta-gradient $\nabla_\theta$ using the analytical formula \eqref{eqn_app:implicit_gradient}.  This problem is a single-task regularization-strength learning problem \citesupp{mackay_practical_1992,bengio_gradient-based_2000,foo_efficient_2007,pedregosa_hyperparameter_2016,lorraine_optimizing_2020} on the Boston housing dataset \citesupp{harrison_jr_hedonic_1978} (70\% learning and 30\% evaluation split). We study a nonlinear neural network model $f_\phi$ with a small hidden layer (20 neurons, hyperbolic tangent transfer function). The bilevel optimization problem we are solving here is the one we consider in Section~\ref{sect:hyperparam_opt}, that is:
\begin{equation}
    \begin{split}
        &\min_\lambda\, \frac{1}{|D^\mathrm{eval}|} \, \sum_{(x,y) \, \in \, D^\mathrm{eval}} l(f_{\phi^*_\lambda}(x), y)\\
        &\mathrm{s.t.} ~ \phi_\lambda^* \in \argmin_\phi \, \frac{1}{|D^\mathrm{learn}|} \, \sum_{(x,y) \,\in\, D^\mathrm{learn}} l(f_\phi(x),y) + \frac{1}{2}\sum_{i=1}^{|\phi|} \lambda_i \phi_i^2.
    \end{split}
\end{equation}

\paragraph{Meta-gradient estimation error.}

We plot the normalized error between the meta-gradient estimate and its true value on Fig~\ref{fig_app:theory_val_boston}. The qualitative behavior closely matches the one we obtained for the quadratic analytical model in the last section, as well as the ones predicted by our theory.

\begin{figure}[htbp]
    \centering
    \includegraphics[width=5.5in]{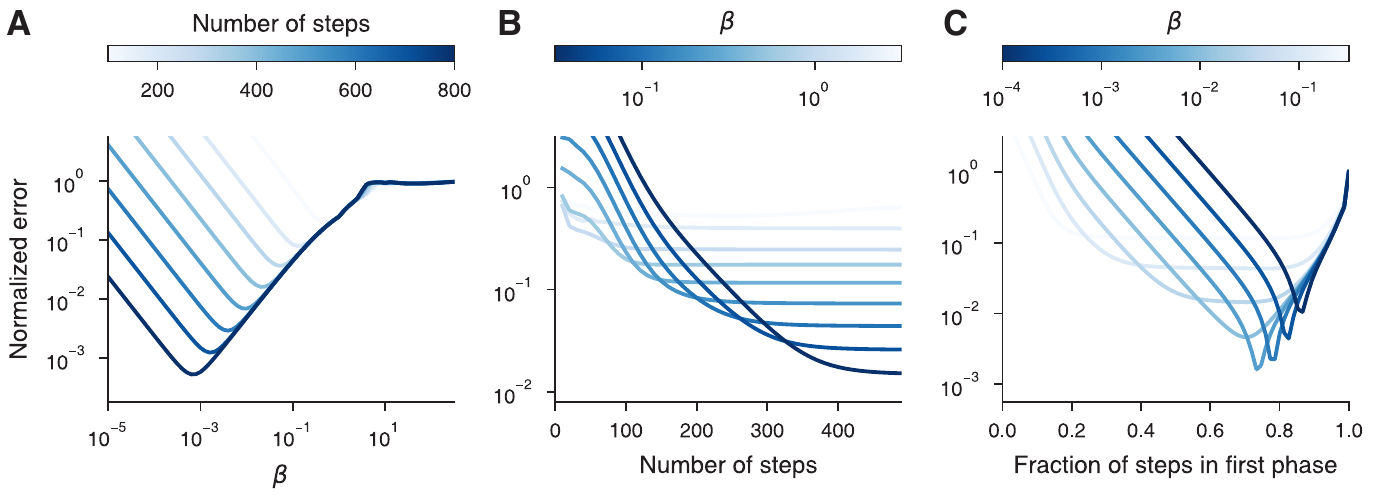}
    \caption{Empirical verification of the theoretical results on a regularization-strength learning problem on the Boston dataset. We plot the normalized error between the meta-gradient estimate $\widehat{\nabla}_\lambda$ and the true one $\nabla_\lambda$, as a function of $\beta$ (A), of the number of steps in the two phases (which is a proxy for $-\log\delta$ and $-\log\delta'$ used in the theory) (B), and as a function of the allocation of the computational resources between the two phases, the total number of steps being fixed to 750 (C).}
    \label{fig_app:theory_val_boston}
\end{figure}

\begin{figure}[htbp]
    \centering
    \includegraphics[width=5.5in]{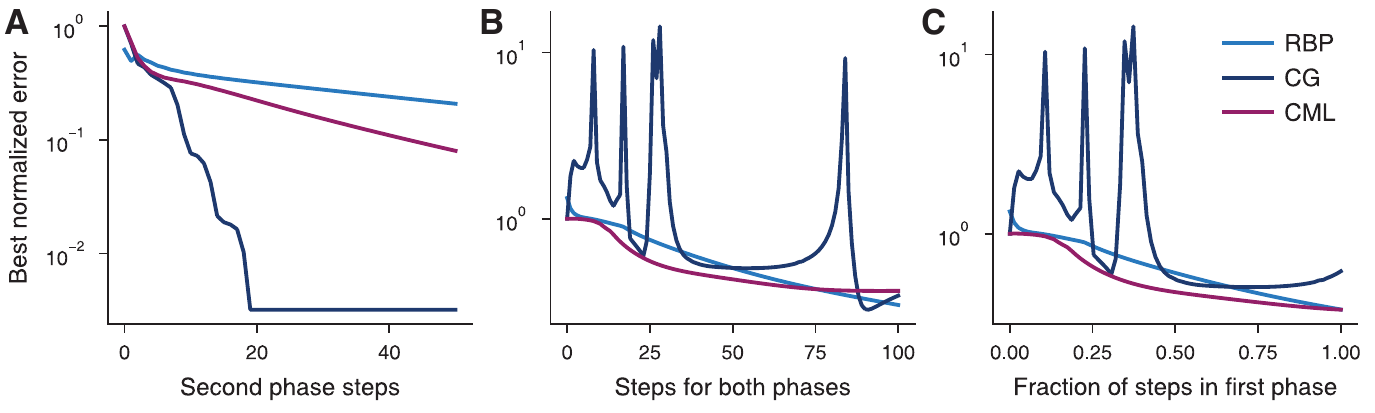}
    \caption{Comparison of the meta-gradient estimation errors provided our contrastive meta-learning rule (CML), recurrent backpropagation (RBP) and conjugate gradients (CG), on a regularization-strength learning problem on the Boston dataset. The hyperparameters for each value of the $x$ axis, such that the normalized error is minimized. (A) Error as a function of the number of steps in the second phase, the first phase being perfectly solved. (B) Error as a function of the number of steps performed in the two phase, which is fixed. (C) Error as a function of the fraction of steps in the first phase, the total number of steps for the two phases being fixed to 75.}
    \label{fig_app:comparison_boston}
\end{figure}

\begin{figure}[htbp]
    \centering
    \includegraphics[width=4.25in]{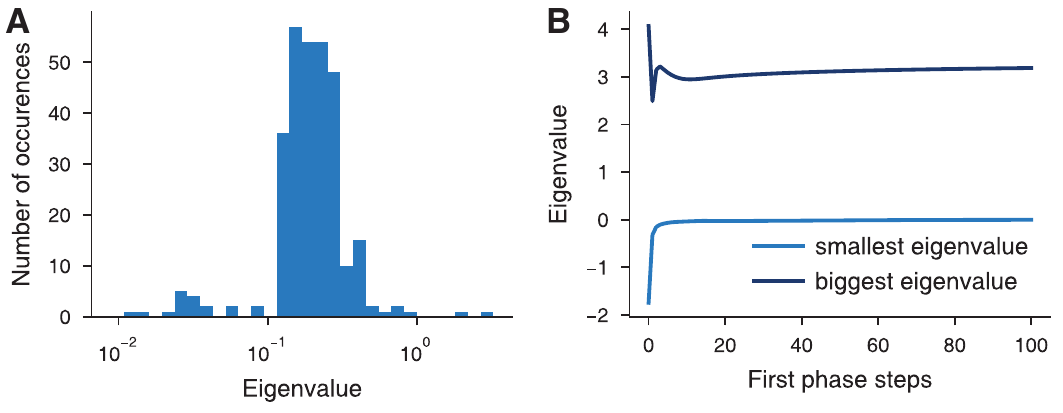}
    \caption{Eigenvalues of the Hessian of the learning loss $\partial_\phi^2 \Li(\hat{\phi}_\lambda, \lambda)$ on the regularization-strength learning problem on the Boston dataset. (A) Spectrum of the Hessian when the first phase is perfectly solved, i.e., $\hat{\phi}_\lambda = \phi^*_\lambda$. (B) Smallest and biggest eigenvalue of the Hessian, as a function of the number of steps in the first phase. The higher the number of steps, the close $\hat{\phi}_\lambda$ is to $\phi^*_\lambda$.}
    \label{fig_app:comparison_boston_eigvals}
\end{figure}

\paragraph{Comparison with other implicit gradient methods.}
\label{sec_app:boston_comparison}

We also use this problem to directly compare the meta-gradient approximation error made by our contrastive meta-learning rule (CML) to other implicit gradient methods, namely recurrent backpropagation (RBP) and conjugate gradient (CG). 
To make the comparison fair, we pick the hyperparameters that yield the smallest error for each method ($\beta$ and the parameters of the optimizer minimizing the second phase for CML, a scaling parameter for RBP, none for CG). Fig~\ref{fig_app:comparison_boston} characterizes the meta-gradient estimation errors by the different methods.

We first perfectly solve the first phase so that $\hat{\phi}_\lambda = \phi^*_\lambda$ and compare how efficient the second phase of those algorithms is. The Hessian of the learning loss $\partial_\phi^2 \Li$ is positive definite as shown in Fig.~\ref{fig_app:comparison_boston_eigvals}A. The conjugate gradient method is therefore much more efficient than the other methods as its assumptions are met, and quickly reaches the numerical accuracy limit. Our rule compares favorably to recurrent backpropagation, even though the theoretical bound is weaker ($\sqrt{\delta'}$ for our rule compared to $\delta'$ for implicit methods \citesupp{pedregosa_hyperparameter_2016}). A possible explanation comes from the fact that we are using Nesterov accelerated gradient descent for the second phase of our contrastive update, whereas the fixed point iteration of RBP is a form of gradient descent.

We repeat our analysis in the more realistic setting in which $\hat{\phi}_\lambda$ is not equal to $\phi^*_\lambda$. We use the same number of steps in the two phases (and the same estimate for $\hat{\phi}_\lambda$ and find that recurrent backpropagation and our contrastive meta-learning rule improve their estimate of the meta-gradient as the number of steps for both phases increases, cf. Fig.~\ref{fig_app:comparison_boston}.B. In contrast, conjugate gradient is unstable when the number of steps is low. In Fig.~\ref{fig_app:comparison_boston_eigvals}B, we check whether the needed assumptions are satisfied by plotting the smallest eigenvalue of the Hessian of the learning loss $\partial_\phi^2 \Li(\hat{\phi}_\lambda, \lambda)$, as a function of the number of steps. We find that this eigenvalue is negative on the range of the number of steps we consider, the Hessian is therefore not positive definite so the conjugate gradient method cannot approximate $\mu$ well. We obtain the same qualitative behavior when we fix the total number of steps, and vary the faction of steps in the first phase, cf. Fig.~\ref{fig_app:comparison_boston}C.

\newpage 
\section{Experimental details}
\label{sec_app:experimental_details}

\subsection{Supervised meta-optimization}
\label{sec_app:supervised_metaopt}

\begin{wraptable}[11]{r}{0.5\textwidth}
\vspace{-0.6cm}
\caption{Meta-learning a per-synapse regularization strength meta-parameter (cf.~Section~\ref{sect:synaptic-model}) on MNIST. Average accuracies (acc.) $\pm$ s.e.m.~over 10 seeds.}
\label{tab:mnist}
\centering
\vspace{0.15cm}
\begin{tabular}{lll}
    \toprule
    Method  & Validation acc. (\%)  & Test acc. (\%)
    \\
    \midrule
    T1-T2 & 98.70$^{\pm 00.08}$ & 97.63$^{\pm 00.03}$\\
    CG & 97.02$^{\pm 00.28}$ & 96.96$^{\pm 00.15}$\\
    RBP & 99.53$^{\pm 00.01}$ & 97.31$^{\pm 00.02}$\\
    CML & 99.45$^{\pm 00.16}$ & 97.92$^{\pm 00.11}$\\
    \bottomrule
\end{tabular}
\end{wraptable}

% Task details
\paragraph{Task details.}
For the supervised meta-optimization experiments we meta-learn parameter-wise l2-regularization strengths ($\omega=0$) on the CIFAR-10 image classification task \citesupp{krizhevsky_learning_2009} starting each learning phase from a fixed neural network initialization. The dataset comprises 60000 32x32 RGB images divided into 10 classes, with 6000 images per class. We split the 50000 training images randomly in half to obtain a training set and a validation set for meta-learning and use the remaining 10000 test images for testing. In Tab.~\ref{tab:mnist} we report additional results on the simpler MNIST image classification task \citesupp{lecun_mnist_1998} for which we use the same data splitting strategy.

% Additional results
\paragraph{Additional results.}
We perform an additional experiment investigating how the number of lower-level parameter updates affects the meta-learning performance of our method and comparison methods.
We consider a simplified data regime for this experiment, using a random subset of 1000 examples of CIFAR-10 split into 50 samples for the learning loss and 950 samples for the evaluation loss which allows us to fit all samples into a single batch during learning and meta-learning.
Results shown in Fig.~\ref{fig:cifar-sm} demonstrate that our contrastive meta-learning rule is able to fit the meta-parameters to the validation set across different number of lower-level parameter updates while competing methods require more updates to obtain similar performance. We found the conjugate gradient method (CG) to be unstable in this setting. To obtain these results, we tuned the hyperparameters for each method for each number of lower-level parameter updates.

\begin{wrapfigure}[20]{r}{0.5\textwidth}
    \vspace{-0.6cm}
    \centering
    \includegraphics{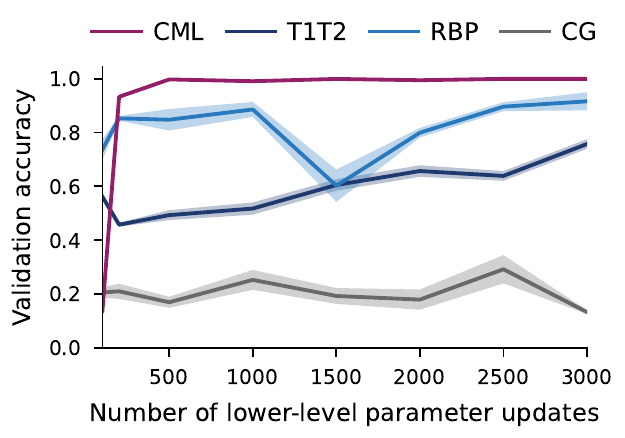}
    \caption{Dependence of the final validation accuracy on the number of lower-level parameter updates obtained after meta-optimizing per-synapse regularization strength meta-parameters on a subset of CIFAR-10. A random subset of 1000 examples from CIFAR-10 are split into 50 samples for the learning loss and 950 samples for the evaluation loss. Mean over 10 random seeds with error bars indicating $\pm1$ s.e.m.}
    \label{fig:cifar-sm}
\end{wrapfigure}

% Architecture details
\paragraph{Architecture details.}
For CIFAR-10 experiments we use a modified version of the classic LeNet-5 model \citesupp{lecun_gradient-based_1998} where we insert batch normalization layers \citesupp{ioffe_batch_2015} before each nonlinearity and replace the hyperbolic tangent nonlinearities with rectified linear units.
For MNIST experiments we use a feedforward neural network with $5$ hidden layers of size $256$ and hyperbolic tangent nonlinearity.

% Comparison method's details
% Optimization details
% Hyperparameter tuning details

\paragraph{Hyperparameters.}
We perform a comprehensive random hyperparameter search for each method with the search space for CIFAR-10 experiments specified in Tab.~\ref{tab:hps-cifar} and the search space for MNIST experiments specified in Tab.~\ref{tab:hps-mnist}.

In Fig.~\ref{fig:theory-vs-CIFAR}B, we furthermore investigate the interaction of $\beta$ and the number of first phase steps on the validation loss, keeping all other hyperparameters fixed. We compare it to the corresponding theoretical prediction visualized in Fig.~\ref{fig_app:bound_viz}B, for the case where $\Lo$ is independent of $\theta$ and $\delta'$ is fixed.

\newpage
\subsection{Few-shot image classification}
\label{sec_app:fewshot}
% Task details
\paragraph{Task details.}
We follow the standard experimental setup  \citesupp{santoro_meta-learning_2016, vinyals_matching_2016, finn_model-agnostic_2017, rajeswaran_meta-learning_2019} for our Omniglot \citesupp{lake_one_2011} and miniImageNet \citesupp{ravi_optimization_2016} experiments.
%Additional results
\paragraph{Additional results.} The results of related work reported in the main text (Tabs.~\ref{tab:miniimagenet}~and~\ref{tab:omniglot}) are taken from the original papers, except for Omniglot first-order MAML which is reported in ref.~\citesupp{nichol_first-order_2018}. In Tab.~\ref{tab:omniglot-sm} we provide results for additional 5-way Omniglot variants that are easier than the 20-way ones studied in the main text.
% Architecture
\paragraph{Architecture details.}
For Omniglot, we use max-pooling instead of stride in the convolutional layers, as we found the latter led to optimization instabilities, as previously reported \citesupp{antoniou_how_2019}. We evaluate the statistics of batch normalization units \citesupp{ioffe_batch_2015} on the test set as in ref.~\citesupp{finn_model-agnostic_2017}, which yields a transductive classifier. More complex architectures whereby a second modulatory neural network which generates task-specific parameters is explicitly modeled \citesupp[a hypernetwork;][]{ha_hypernetworks_2017,rusu_meta-learning_2019,zhao_meta-learning_2020} can be easily accommodated into our framework, but here for simplicity we implement our top-down modulation model by taking advantage of existing batch normalization layers in our neural networks and consider the gain and shift parameters of these units as well as the synaptic weights and biases of the output layer as our task-specific parameters $\phi$. 
% Optimization details
\paragraph{Optimization details.}
We used the symmetric version of our contrastive rule for meta-learning and the Kaiming scheme for parameter initialization \citesupp{he_delving_2015}.
The task-specific learning and evaluation losses are both taken to be the cross-entropy with dataset splits into learning and evaluation data following the setup considered by \citet{finn_model-agnostic_2017}. 
In order to stabilize results, we used Polyak averaging \citesupp{polyak_acceleration_1992} for the meta-parameters to compute final performance. Specifically, we started averaging meta-parameters after a certain number of meta-parameter updates (5 for Omniglot, 50 for miniImageNet). Note that the performance of the non-averaged meta-parameters performs only slightly differently averaged over iterations but is considerably more noisy.
\paragraph{Hyperparameters.}
We perform a comprehensive grid-search over hyperparameters with search ranges and optimal hyperparameters found reported in Tab.~\ref{tab:hps-omni-image}.

\begin{table}[htbp!]
\caption{Few-shot learning of Omniglot characters. We report results obtained with contrastive meta-learning for the synaptic and modulatory models. We present test set classification accuracy (\%)  averaged over 5 seeds $\pm$ std.}
\label{tab:omniglot-sm}
\centering
\begin{tabular}{lllll}
    \toprule
    Method    & 5-way 1-shot & 5-way 5-shot
              & 20-way 1-shot & 20-way 5-shot
    \\
    \midrule
    MAML \citesupp{finn_model-agnostic_2017} & 98.7$^{\pm 0.4}$ & 99.9$^{\pm0.1}$ & 95.8$^{\pm0.3}$ & 98.9$^{\pm0.2}$\\
    First-order MAML \citesupp{finn_model-agnostic_2017}  & 98.3$^{\pm 0.5}$ & 99.2$^{\pm0.2}$ & 89.4$^{\pm0.5}$ & 97.9$^{\pm0.1}$\\
    Reptile \citesupp{nichol_first-order_2018} & 97.68$^{\pm0.04}$     & 99.48$^{\pm0.06}$  & 89.43$^{\pm0.14}$ & 97.12$^{\pm0.32}$\\
    \midrule
    iMAML \citesupp{rajeswaran_meta-learning_2019} & 99.16$^{\pm0.35}$ &99.67$^{\pm0.12}$ &94.46$^{\pm0.42}$ &98.69$^{\pm0.1}$\\
    \midrule
    CML (synaptic) & 98.11$^{\pm0.34}$ & 99.49$^{\pm0.16}$ & 94.16$^{\pm0.12}$ & 98.06$^{\pm0.26}$\\
    CML (modulatory)  & 98.05$^{\pm0.06}$ & 99.45$^{\pm0.04}$  & 94.24$^{\pm0.39}$  & 98.60$^{\pm0.27}$  \\
    \bottomrule
  \end{tabular}
\end{table}

\newpage
\subsection{Few-shot regression in recurrent spiking network}
\label{sec_app:spiking}

% Task details
\paragraph{Task details.}
We consider a standard sinusoidal 10-shot regression problem. For each task a sinusoid with random amplitude sampled uniformly from $[0.1, 5.0]$ and random phase sampled uniformly from $[0, \pi]$ is generated. 10 data points are drawn uniformly from the range $[-5, 5]$ both for the learning loss and for the evaluation loss.

% Additional results

% Architecture details
\paragraph{Architecture details.}
We encode the input with a population of $100$ neurons similar to \citet{bellec_long_2018}. Each neuron $i$ has a Gaussian response field with the mean values $\mu_i$ evenly distributed in the range $[0, 1]$ across neurons and a fixed variance $\sigma^2=0.0002$. The firing probability of each neuron at a single time step is given by $p_i = \exp(\frac{- (\mu_i - z)^2}{ 2 \sigma^2}) $ where $z$ is the input value standardized to the range $[0, 1]$. We generate $20$ time steps for each data point by sampling spikes from a Bernoulli distribution given the firing probabilities $p_i$ for each neuron.

We use a singe-layer recurrent spiking neural network with leaky integrate and fire neurons that follow the time-discretized dynamics with step size $\Delta t = 1.0$ (notation taken from \citet{bellec_solution_2020}):
\begin{align}
    h_j^{t+1} &= \alpha h_j^{t} + \sum_{i \neq j} W_{ji}^{\text{rec}} z_i^{t} + \sum_{i} W_{ji}^{\text{in}} x_i^{t+1} - z_j^{t} v_{\text{th}} \\
    z_j^{t} &= \Theta(h_j^t - v_{\text{th}}) \\
    y_k^{t+1} &= \kappa y_k^{t} + \sum_{j} W_{kj}^{\text{out}} z_j^{t}
\end{align}
where $W^{\text{in}}, W^{\text{rec}}, W^{\text{out}}$ are the synaptic input, recurrent and output weights, $\alpha = \exp(-\frac{\Delta t}{\tau_{\text{hidden}}})$ and $\kappa = \exp(-\frac{\Delta t}{\tau_{\text{out}}})$ are decay factors with $\tau_{\text{hidden}} = \tau_{\text{out}} = 30.0$ , $v_{\text{th}}=0.1$ is the threshold potential, and $\Theta(\cdot)$ denotes the Heaviside step function. The weights are initialized using the Kaiming normal scheme \citesupp{he_delving_2015} and scaled down by a factor of $0.1, 0.01, 0.1$ for $W^{\text{in}}, W^{\text{rec}}, W^{\text{out}}$ respectively.

% Optimization details
\paragraph{Optimization details.}
The weights are updated according to e-prop \citesupp{bellec_solution_2020}:
\begin{align}
  \Delta W_{kj}^{\text{out}} &\propto \sum_t (y_k^{*,t} - y_k^{t}) \sum_{t' \leq t} (\kappa^{t-t'} z_j^{t'}) \\
  \Delta W_{ji}^{\text{rec}} &\propto \sum_t ( \sum_k W_{kj}^{\text{out}} (y_k^{*,t} - y_k^{t})) \sum_{t' \leq t} (\kappa^{t-t'} h_j^{t'} \sum_{t'' \leq t'}(\alpha^{t'-t''} z_i^{t''}))\\
  \Delta W_{ji}^{\text{in}} &\propto \sum_t ( \sum_k W_{kj}^{\text{out}} (y_k^{*,t} - y_k^{t})) \sum_{t' \leq t} (\kappa^{t-t'} h_j^{t'} \sum_{t'' \leq t'}(\alpha^{t'-t''} x_i^{t''}))
\end{align}

The loss is computed as the mean-squared error between the target and the prediction given by the average output over time. Note that the output of the network is non-spiking. We add a regularization term to the loss that is computed as the mean squared difference between the average neuron firing rate and a target rate and decrease the learning rate for updating $W^{\text{out}}$ with e-prop by a factor of 0.1.
We use the symmetric version of our contrastive rule to obtain meta-updates.

% Comparison method's details
\paragraph{Comparison methods.}
We compare our method to a standard baseline where both fast parameter updates and slow meta-parameter updates are computed by backpropagating through the synaptic plasticity process (BPTT+BPTT) using surrogate gradients to handle spiking nonlinearities \citesupp{neftci_surrogate_2019}. As this biologically-implausible process is computationally expensive, we restrict the number of update steps on the learning loss to 10 changes as done by prior work \citesupp{finn_model-agnostic_2017}. For a second comparison method, we compute the fast parameter updates using the e-prop update stated above and use backpropagation through 10 e-prop updates for meta-parameter updates (BPTT+e-prop).

% Hyperparameter tuning details
\paragraph{Hyperparameters.}
For each method we employ an extensive random hyperparameter search over the search space defined in Tab.~\ref{tab:hps-spiking} using a meta-validation set to select the optimal set of hyperparameters.

\newpage
\subsection{Meta-reinforcement learning}
\label{sec_app:metaRL}

\begin{wrapfigure}[20]{r}{0.38\textwidth}
    \vspace{-0.5cm}
    \centering
    \includegraphics[width=0.33\textwidth]{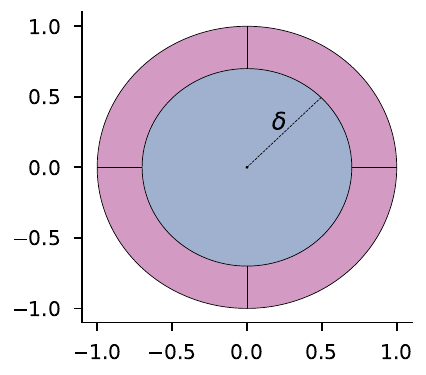}
    \caption{The wheel bandit task tiles the context space into an inner low-reward region (blue) and a high-reward outer rim (purple). Across tasks the radius $\delta$ of the inner low-reward region is varied. The high-reward region is divided into 4 quadrants, depending on which the optimal action changes.}
    \label{fig:wheel}
\end{wrapfigure}

% Task details
\paragraph{Task details.}
The contextual wheel bandit was introduced by \citet{riquelme_deep_2018} to parametrize the task difficulty of a contextual bandit task in terms of its exploration-exploitation trade-off. Each task consists of a sequence of context coordinates $X$ randomly drawn from the unit circle and a scalar radius $\delta \in [0,1]$. The radius $\delta$ tiles the unit circle into a low-reward region and a high-reward region, see Fig.~\ref{fig:wheel}. If the current context lies within the low-reward region, $\lVert X \rVert \le \delta$, all actions $a \in \{1,2,3,4\}$ return reward $ r \sim \mathcal{N}(1.0, 0.01^2)$ except for the last action $a=5$ which returns $ r \sim \mathcal{N}(1.2, 0.01^2)$ and is thus optimal. If the current context lies within the high-reward region, one of the first four actions is optimal returning a high reward $r \sim \mathcal{N}(50.0, 0.01^2)$, the last arm still returns $ r \sim \mathcal{N}(1.2, 0.01^2)$ and the remaining arms return $ r \sim \mathcal{N}(1.0, 0.01^2)$. Which of the first four actions returns the high reward depends on the quadrant of the high-reward region in which the current context lies. Action 1 is optimal in the upper right quadrant, action 2 in the lower right quadrant, action 3 in the upper left quadrant and action 4 in the lower left quadrant.

% Additional results
\paragraph{Additional results.}
Following previous work \citesupp{garnelo_neural_2018, ravi_amortized_2019}, we treat the contextual wheel bandit as a meta-learning problem. During meta-learning, we sample $M=64$ different wheel tasks $\{\delta_i\}_{i=1}^{M}$, $\delta_i \sim \mathcal{U}(0,1)$ for each of which we sample a sequence of $N=562$ random contexts, random actions and corresponding rewards $\{(X_j, a_j, r_j)\}_{j=1}^N$. We use $512$ observations for the learning loss and $50$ observations for the evaluation loss.
Both the learning $l^\mathrm{learn}$ and the evaluation loss $l^\mathrm{eval}$ are measured as the mean-squared error between observed reward and the predicted value for the corresponding action.
For each meta-learning step, we randomly sample from the $M$ tasks such that specific tasks may be encountered multiple times. After meta-learning, we evaluate the agent online on a long episode with $80000$ contexts and track its cumulative reward relative to the cumulative reward obtained by an agent that chooses its actions at random. As done by \citet{riquelme_deep_2018}, each action is initially explored twice before choosing actions according to the agent's policy. The extended results over more settings for $\delta$ can be seen in Tab.~\ref{tab:bandit-sm}. Results for the NeuralLinear baseline reported here and in the main text are taken from the original paper \citesupp{riquelme_deep_2018}.

\begin{table}[b]
\caption{Cumulative regret on the wheel bandit problem for different values of $\delta$. Values are normalized by the cumulative regret of a uniformly random agent. Averages over 50 seeds $\pm$ s.e.m.}
\label{tab:bandit-sm}
\centering
\begin{tabular}{llllll}
\toprule
$\delta$     & 0.5               & 0.7               & 0.9               & 0.95              & 0.99               \\
\midrule
NeuralLinear \citesupp{riquelme_deep_2018} & 0.95$^{\pm 0.02}$ & 1.60$^{\pm 0.03}$ & 4.65$^{\pm 0.18}$ & 9.56$^{\pm 0.36}$ & 49.63$^{\pm 2.41}$ \\ \midrule
MAML                                    & 0.45$^{\pm 0.01}$ & 0.62$^{\pm 0.03}$ & 1.02$^{\pm 0.76}$ & 1.56$^{\pm 0.62}$ & 15.21$^{\pm 1.69}$ \\ 
%MAML Linear Hnet                        & 0.29              & 0.49              & 0.98              & 1.41              & 9.13               \\ 
%MAML Gain Mod                           & 0.29              & 0.43              & 0.88              & 1.38              & 5.43               \\ \midrule
CML (synaptic)                          & 0.40$^{\pm 0.02}$ & 0.45$^{\pm 0.01}$ & 0.82$^{\pm 0.02}$ & 1.42$^{\pm 0.07}$ & 12.27$^{\pm 1.02}$ \\ 
CML (modulatory)                        & 0.42$^{\pm 0.01}$ & 0.65$^{\pm 0.03}$ & 1.83$^{\pm 0.11}$ & 3.68$^{\pm 0.59}$ & 16.46$^{\pm 1.80}$ \\ 
\bottomrule
\end{tabular}
\end{table}

% Architecture details
\paragraph{Architecture results.}
For all methods including the synaptic consolidation and modulatory network model, we consider as a base a multilayer perceptron with ReLU nonlinearites and two hidden layers with 100 units. For the modulatory network model, each hidden unit is multiplied by a gain and shifted by a bias prior to applying the nonlinearity. The non-meta-learned baseline NeuralLinear we report from \citet{riquelme_deep_2018} additionally uses a Bayesian regression head for each action and applies Thompson sampling to choose actions.

% Optimization and evaluation details
\paragraph{Optimization and evaluation details.}
During online evaluation, we take the greedy action with respect to the predicted expected rewards on each context and store each observation $(X_j, a_j, r_j)$ in a replay buffer. This data is used to train the fast parameters every $t_f$ contexts for $t_s$ steps, where $t_f, t_s$ are hyperparameters tuned for every method. 

% Hyperparameter tuning details
\paragraph{Hyperparameters.}
We perform a comprehensive random hyperparameter search for each method with the search space specified in Tab.~\ref{tab:hps-bandit}. Optimal parameters are selected on 5 validation tasks with $\delta=0.95$.

% Hyperparameters

% Supervised meta-optimization

\begin{landscape}
\begin{table}[bt]
\caption{Hyperparameter search space for the supervised meta-optimization experiment on CIFAR-10. For all methods 500 samples were randomly drawn from the search space and Asynchronous HyperBand from ray tune \protect\citesupp{liaw_tune_2018} was used for scheduling with a grace period of 10. Best found parameters are marked in bold.}
\label{tab:hps-cifar}
\centering
\begin{tabular}{@{}l L{8.5cm} L{7cm} @{}}
\toprule
Hyperparameter            & CML                                                          & CG                                                           \\ \midrule
\texttt{batch\_size}      & 500                                                          & 500                                                          \\
$\beta$                   & $\{0.01, 0.03, 0.1, 0.3, 1.0, 3.0, \mathbf{10.0}\}$          & -                                                            \\
$\lambda$                 & $\{10^{-5}, 10^{-4}, 10^{-3}, 10^{-2}, \mathbf{10^{-1}}\}$   & $\{\mathbf{10^{-5}}, 10^{-4}, 10^{-3}, 10^{-2}, 10^{-1}\}$   \\
\texttt{lr\_inner}        & $\{0.0001, \mathbf{0.0003}, 0.001,0.003, 0.01, 0.03, 0.1\}$  & $\{0.0001, 0.0003, \mathbf{0.001}, 0.003, 0.01, 0.03, 0.1\}$ \\
\texttt{lr\_nudged}       & $\{\mathbf{0.0001}, 0.0003, 0.001, 0.003, 0.01, 0.03, 0.1\}$ & -                                                            \\
\texttt{lr\_outer}        & $\{0.0001, 0.0003, 0.001, 0.003, 0.01, \mathbf{0.03}, 0.1\}$ & $\{0.0001, 0.0003, \mathbf{0.001}, 0.003, 0.01, 0.03, 0.1\}$ \\
\texttt{optimizer\_inner} & \{adam, \textbf{sgd\_nesterov\_0.9}\}                        & \{adam, \textbf{sgd\_nesterov\_0.9}\}                        \\
\texttt{optimizer\_outer} & adam                                                         & adam                                                         \\
\texttt{steps\_cg}        & -                                                            & $\{ \mathbf{100}, 500, 1000, 2000 \}$                        \\
\texttt{steps\_inner}     & $\{ 2000, 3000, \mathbf{5000} \}$                            & $\{ \mathbf{2000}, 3000, 5000 \}$                            \\
\texttt{steps\_nudged}    & $\{ 100, \mathbf{200}, 500 \}$                               & -                                                            \\
\texttt{steps\_outer}     & 100                                                          & 100                                                          \\ \bottomrule
\end{tabular}
\\ [1em]
\begin{tabular}{@{}l L{8.5cm} L{7cm} @{}}
\toprule
Hyperparameter            & NSA                                                             & T1T2                                                         \\ \midrule
\texttt{batch\_size}      & 500                                                             & 500                                                          \\
$\lambda$                 & $\{ \mathbf{10^{-5}}, 10^{-4}, 10^{-3}, 10^{-2}, 10^{-1}\}$     & $\{10^{-5}, \mathbf{10^{-4}}, 10^{-3}, 10^{-2}, 10^{-1}\}$   \\
\texttt{lr\_inner}        & $\{0.0001, 0.0003, \mathbf{0.001}, 0.003, 0.01, 0.03, 0.1\}$    & $\{0.0001, 0.0003, \mathbf{0.001}, 0.003, 0.01, 0.03, 0.1\}$ \\
\texttt{lr\_outer}        & $\{0.0001, 0.0003, 0.001, \mathbf{0.003}, 0.01, 0.03, 0.1\}$    & $\{0.0001, 0.0003, \mathbf{0.001}, 0.003, 0.01, 0.03, 0.1\}$ \\
\texttt{nsa\_alpha}       & $ \{ 0.000001, 0.000003, 0.00001, 0.00003, 0.0001, \mathbf{0.0003} \} $ & -                                                            \\
\texttt{optimizer\_inner} & \{adam, \textbf{sgd\_nesterov\_0.9}\}                           & \{adam, \textbf{sgd\_nesterov\_0.9}\}                        \\
\texttt{optimizer\_outer} & adam                                                            & adam                                                         \\
\texttt{steps\_inner}     & $\{ \mathbf{2000}, 3000, 5000 \}$                               & $\{ 2000, \mathbf{3000}, 5000 \}$                            \\
\texttt{steps\_nsa}       & 500                                                             & -                                                            \\
\texttt{steps\_outer}     & 100                                                             & 100                                                          \\ \bottomrule
\end{tabular}
\\ [1em]
\begin{tabular}{@{}l L{10.5cm} L{7cm} @{}}
\toprule
Hyperparameter            & TBPTL                                                          & no-meta                                                           \\ \midrule
\texttt{batch\_size}      & 500                                                          & 500                                                          \\
$\lambda$                 & $\{10^{-5}, \mathbf{10^{-4}}, 10^{-3}, 10^{-2}, 10^{-1}\}$            & $\{10^{-5}, 10^{-4}, 10^{-3}, \mathbf{10^{-2}}, 10^{-1}\}$              \\
\texttt{lr\_inner}        & $\{0.0001, 0.0003, \mathbf{0.001},0.003, 0.01, 0.03, 0.1\}$  & $\{0.0001, 0.0003, \mathbf{0.001}, 0.003, 0.01, 0.03, 0.1\}$ \\
\texttt{lr\_outer}        & $\{0.0001, 0.0003, \mathbf{0.001}, 0.003, 0.01, 0.03, 0.1\}$ & -  \\
\texttt{optimizer\_inner} & \{adam, \textbf{sgd\_nesterov\_0.9}\}                        & \{adam, \textbf{sgd\_nesterov\_0.9}\}                        \\
\texttt{optimizer\_outer} & adam                                                         & -                                                         \\
\texttt{steps\_inner}     & \{2000, \textbf{3000}, 5000\}                                         & 5000                                                         \\
\texttt{steps\_outer}     & 100                                                          & 0                                                          \\ \bottomrule
\end{tabular}

\end{table}
\end{landscape}

\begin{landscape}
\begin{table}[bt]
\caption{Hyperparameter search space for the supervised meta-optimization experiment on MNIST. For all methods 500 samples were randomly drawn from the search space and Asynchronous HyperBand from ray tune \protect\citesupp{liaw_tune_2018} was used for scheduling with a grace period of 10. Best found parameters are marked in bold.}
\label{tab:hps-mnist}
\centering
\begin{tabular}{@{}l L{10.5cm} L{7cm} @{}}
\toprule
Hyperparameter            & CML                                                          & CG                                                           \\ \midrule
\texttt{batch\_size}      & 500                                                          & 500                                                          \\
$\beta$                   & $\{0.01, 0.03, 0.1, 0.3, 1.0, \mathbf{3.0}, 10.0\}$          & -                                                            \\
$\lambda$                 & $\{\mathbf{0.00001}, 0.0001, 0.001, 0.01, 0.1\}$             & $\{0.00001, \mathbf{0.0001}, 0.001, 0.01, 0.1\}$             \\
\texttt{lr\_inner}        & $\{0.0001, \mathbf{0.0003}, 0.001,0.003, 0.01, 0.03, 0.1\}$  & $\{\mathbf{0.0001}, 0.0003, 0.001, 0.003, 0.01, 0.03, 0.1\}$ \\
\texttt{lr\_nudged}       & $\{0.0001, \mathbf{0.0003}, 0.001,0.003, 0.01, 0.03, 0.1\}$  & -                                                            \\
\texttt{lr\_outer}        & $\{\mathbf{0.0001}, 0.0003, 0.001, 0.003, 0.01, 0.03, 0.1\}$ & $\{0.0001, \mathbf{0.0003}, 0.001,0.003, 0.01, 0.03, 0.1\}$  \\
\texttt{optimizer\_inner} & \{\textbf{adam}, sgd\_nesterov\_0.9\}                        & \{\textbf{adam}, sgd\_nesterov\_0.9\}                        \\
\texttt{optimizer\_outer} & adam                                                         & adam                                                         \\
\texttt{steps\_cg}        & -                                                            & $\{ \mathbf{100}, 250, 500, 1000, 2000 \}$                   \\
\texttt{steps\_inner}     & 2000                                                         & 2000                                                         \\
\texttt{steps\_nudged}    & $\{ 100, \mathbf{200}, 500 \}$                               & -                                                            \\
\texttt{steps\_outer}     & 100                                                          & 100                                                          \\ \bottomrule
\end{tabular}
\\ [1em]
\begin{tabular}{@{}l L{10.5cm} L{7cm} @{}}
\toprule
Hyperparameter            & NSA                                                                                  & T1T2                                                         \\ \midrule
\texttt{batch\_size}      & 500                                                                                  & 500                                                          \\
$\lambda$                 & $\{0.00001, \mathbf{0.0001}, 0.001, 0.01, 0.1\}$                                     & $\{\mathbf{0.00001}, 0.0001, 0.001, 0.01, 0.1\}$             \\
\texttt{lr\_inner}        & $\{0.0001, 0.0003, \mathbf{0.001}, 0.003, 0.01, 0.03, 0.1\}$                         & $\{0.0001, 0.0003, 0.001, 0.003, \mathbf{0.01}, 0.03, 0.1\}$ \\
\texttt{lr\_outer}        & $\{0.0001, 0.0003, 0.001, 0.003, \mathbf{0.01}, 0.03, 0.1\}$                         & $\{\mathbf{0.0001}, 0.0003, 0.001, 0.003, 0.01, 0.03, 0.1\}$ \\
\texttt{nsa\_alpha}       & $ \{ 0.000001, 0.000003, 0.00001, 0.00003, 0.0001, 0.0003, \mathbf{0.001}, 0.003 \}$ & -                                                            \\
\texttt{optimizer\_inner} & \{adam, \textbf{sgd\_nesterov\_0.9}\}                                                & \{adam, \textbf{sgd\_nesterov\_0.9}\}                        \\
\texttt{optimizer\_outer} & adam                                                                                 & adam                                                         \\
\texttt{steps\_inner}     & 2000                                                                                 & 2000                                                         \\
\texttt{steps\_nsa}       & 200                                                                                  & -                                                            \\
\texttt{steps\_outer}     & 100                                                                                  & 100                                                          \\ \bottomrule
\end{tabular}
\end{table}
\end{landscape}

\begin{landscape}
\begin{table}[bt]
\begin{center}
\caption{Hyperparameter search space for the few-shot image classification experiments on Omniglot and  miniImageNet. Best found parameters are marked in bold.}
\label{tab:hps-omni-image}
\begin{tabular}{@{}llllll@{}}
\toprule
Hyperparameter            & Omni-5W-1s                           & Omni-5W-5s                           & Omni-20W-1s                           & Omni-20W-5s                           & miniImageNet                          \\ \midrule
\texttt{batch\_size}      & 32                                   & 32                                   & 16                                    & 16                                    & 4                                     \\
$\beta$                   & \{0.01, 0.03,\textbf{0.1}, 0.3, 1.\} & \{0.01, 0.03,\textbf{0.1}, 0.3, 1.\} & \{0.01, \textbf{0.03}, 0.1, 0.3, 1.\} & \{0.01, 0.03, \textbf{0.1}, 0.3, 1.\} & \{\textbf{0.01}, 0.03, 0.1, 0.3, 1.\} \\
$\lambda$                 & \{0.1, 0.25,\textbf{0.5}\}           & \{\textbf{0.1}, 0.25, 0.5\}          & \{0.1, 0.25, \textbf{0.5}\}           & \{0.1, \textbf{0.25}, 0.5\}           & \{0.1, 0.25, \textbf{0.5}\}           \\
\texttt{lr\_inner}        & 0.01                                 & 0.01                                 & 0.01                                  & 0.01                                  & 0.01                                  \\
\texttt{lr\_outer}        & \{\textbf{0.01}, 0.001 \}            & \{\textbf{0.01}, 0.001 \}            & \{\textbf{0.01}, 0.001 \}             & \{\textbf{0.01}, 0.001 \}             & 0.001                                 \\
\texttt{optimizer\_inner} & gd\_nesterov\_0.9                    & gd\_nesterov\_0.9                    & gd\_nesterov\_0.9                     & gd\_nesterov\_0.9                     & gd\_nesterov\_0.9                     \\
\texttt{optimizer\_outer} & adam                                 & adam                                 & adam                                  & adam                                  & adam                                  \\
\texttt{steps\_inner}     & \{50, 100, 150, \textbf{200}\}       & \{50, 100, 150, \textbf{200}\}       & \{50, 100, 150, \textbf{200}\}        & \{50, 100, 150, \textbf{200}\}        & \{50, \textbf{100}, 150, 200\}        \\
\texttt{steps\_nudged}    & \{50, \textbf{100}, 150, 200\}       & \{50, \textbf{100}, 150, 200\}       & \{50, \textbf{100}, 150, 200\}        & \{50, 100, 150, \textbf{200}\}        & \{25, 50, \textbf{75}, 100\}          \\
\texttt{steps\_outer}     & 3750                                 & 3750                                 & 3750                                  & 3750                                  & 25000                                 \\ \bottomrule
\end{tabular}
\end{center}
\end{table}
\end{landscape}
% Few-shot regression in recurrent spiking network

\begin{landscape}
\begin{table}[bt]
\caption{Hyperparameter search space for the sinusoidal fewshot regression experiment. For all methods 500 samples were randomly drawn from the search space and the Asynchronous HyperBand scheduler from ray tune was used with a grace period of 10 \protect\citesupp{liaw_tune_2018}. Best found parameters are marked in bold.}
\label{tab:hps-spiking}
\centering
\begin{tabular}{@{}ll@{}}
\toprule
Hyperparameter                   & CML + e-prop                                                          \\ \midrule
\texttt{activity\_reg\_strength} & $ \{ 10^{-1}, 10^{-2},10^{-3},10^{-4},\mathbf{10^{-5}},10^{-6}\} $    \\
\texttt{activity\_reg\_target}   & $\{0.05, 0.1, \mathbf{0.2}\}$                                         \\
\texttt{batch\_size}             & $\{1,5,\mathbf{10}\}$                                                 \\
$\beta$                          & $\{0.01, 0.03, 0.1, 0.3, 1.0, \mathbf{3.0}, 10.0\}$                   \\
$\lambda$                        & $\{10^{0}, 10^{-1}, \mathbf{10^{-2}},10^{-3},10^{-4},10^{-5},10^{-6} \}$ \\
\texttt{lr\_inner}               & $\{0.0001, 0.0003, \mathbf{0.001}, 0.003, 0.01, 0.03, 0.1\}$          \\
\texttt{lr\_nudged}              & $\{0.0001, 0.0003, 0.001, \mathbf{0.003}, 0.01, 0.03, 0.1\}$          \\
\texttt{lr\_outer}               & $\{0.0001, 0.0003, 0.001, \mathbf{0.003}, 0.01, 0.03, 0.1\}$          \\
\texttt{meta\_batch\_size}       & $\{1, 10, \mathbf{25}\}$                                              \\
\texttt{optimizer\_inner}        & \{\textbf{adam}, sgd\_nesterov\_0.9\}                                 \\
\texttt{optimizer\_outer}        & adam                                                                  \\
\texttt{steps\_inner}            & 500                                                                   \\
\texttt{steps\_nudged}           & $\{50, \mathbf{100}, 200\}$                                           \\
\texttt{steps\_outer}            & 1000                                                                  \\ \bottomrule
\end{tabular}
\\[1em]
\begin{tabular}{@{}llll@{}}
\toprule
Hyperparameter                   & BPTT + e-prop                                                      & BPTT + BPTT                                                        & TBPTL + e-prop                                                     \\ \midrule
\texttt{activity\_reg\_strength} & $ \{ 10^{-1}, \mathbf{10^{-2}},10^{-3},10^{-4},10^{-5},10^{-6}\} $ & $ \{ 10^{-1}, 10^{-2},10^{-3},10^{-4},10^{-5},\mathbf{10^{-6}}\} $ & $ \{ 10^{-1}, 10^{-2},10^{-3},10^{-4},10^{-5},\mathbf{10^{-6}}\} $ \\
\texttt{activity\_reg\_target}   & $\{\mathbf{0.05}, 0.1, 0.2\}$                                      & $\{0.05, \mathbf{0.1}, 0.2\}$                                      & $\{0.05, \mathbf{0.1}, 0.2\}$                                      \\
\texttt{batch\_size}             & $\{\mathbf{1},5,10\}$                                              & $\{1,5,\mathbf{10}\}$                                              & $\{\mathbf{1},5,10\}$                                              \\
\texttt{lr\_inner}               & $\{\mathbf{0.0001}, 0.0003, 0.001, 0.003, 0.01, 0.03, 0.1\}$       & $\{0.0001, \mathbf{0.0003}, 0.001, 0.003, 0.01, 0.03, 0.1\}$       & $\{\mathbf{0.0001}, 0.0003, 0.001, 0.003, 0.01, 0.03, 0.1\}$       \\
\texttt{lr\_outer}               & $\{0.0001, \textbf{0.0003}, 0.001, 0.003, 0.01, 0.03, 0.1\}$       & $\{0.0001, \mathbf{0.0003}, 0.001, 0.003, 0.01, 0.03, 0.1\}$       & $\{0.0001, 0.0003, \mathbf{0.001}, 0.003, 0.01, 0.03, 0.1\}$       \\
\texttt{meta\_batch\_size}       & $\{1, \mathbf{10}, 25\}$                                           & $\{1, \mathbf{10}, 25\}$                                           & $\{\mathbf{1},10, 25\}$                                            \\
\texttt{optimizer\_inner}        & \{adam, \textbf{sgd\_nesterov\_0.9}\}                              & \{adam, \textbf{sgd\_nesterov\_0.9}\}                              & sgd                                                                \\
\texttt{optimizer\_outer}        & adam                                                               & adam                                                               & adam                                                               \\
\texttt{steps\_inner}            & 10                                                                 & 10                                                                 & 500                                                                \\
\texttt{steps\_outer}            & 1000                                                               & 1000                                                               & 1000                                                               \\ \bottomrule
\end{tabular}
\end{table}
\end{landscape}

\begin{landscape}
\begin{table}[bt]
\begin{center}
\caption{Hyperparameter search space for the wheel bandit experiment. For all methods 1000 samples were randomly drawn from the search space. Best found parameters are marked in bold.}
\label{tab:hps-bandit}
\begin{tabular}{@{}llll@{}}
\toprule
Hyperparameter             & CML (synaptic)                                               & CML (modulatory)                                             & MAML                                                         \\ \midrule
\texttt{batch\_size}       & 512                                                          & 512                                                          & 512                                                          \\
$\beta$                    & $\{0.01, 0.03, 0.1, \mathbf{0.3}, 1.0, 3.0, 10.0\}$          & $\{0.01, 0.03, 0.1, 0.3, 1.0, 3.0, \mathbf{10.0}\}$          & -                                                            \\
$\lambda$                  & $\{10^{-6}, 10^{-5}, \dots, \mathbf{10^{3}}\}$               & -                                                            & -                                                            \\
\texttt{lr\_inner}         & $\{\mathbf{0.0001}, 0.0003, 0.001,0.003, 0.01, 0.03, 0.1\}$  & $\{0.0001, \mathbf{0.0003}, 0.001, 0.003, 0.01, 0.03, 0.1\}$ & $\{0.0001, 0.0003, 0.001, 0.003, \mathbf{0.01}, 0.03, 0.1\}$ \\
\texttt{lr\_nudged}        & $\{0.0001, 0.0003, 0.001, 0.003, 0.01, \mathbf{0.03}, 0.1\}$ & $\{\mathbf{0.0001}, 0.0003, 0.001, 0.003, 0.01, 0.03, 0.1\}$ & -                                                            \\
\texttt{lr\_online}        & $\{\mathbf{0.0001}, 0.0003, 0.001,0.003, 0.01, 0.03, 0.1\}$  & $\{\mathbf{0.0001}, 0.0003, 0.001, 0.003, 0.01, 0.03, 0.1\}$ & $\{0.0001, 0.0003, \mathbf{0.001}, 0.003, 0.01, 0.03, 0.1\}$ \\
\texttt{lr\_outer}         & $\{0.0001, 0.0003, 0.001, 0.003, 0.01, \mathbf{0.03}, 0.1\}$ & $\{0.0001, 0.0003, 0.001, 0.003, 0.01, \mathbf{0.03}, 0.1\}$ & $\{0.0001, 0.0003, 0.001, 0.003, 0.01, 0.03, \mathbf{0.1}\}$ \\
\texttt{meta\_batch\_size} & $\{\mathbf{8}, 16, 32\}$                                     & $\{8, \mathbf{16}, 32\}$                                     & $\{8, 16, \mathbf{32}\}$                                     \\
\texttt{optimizer\_inner}  & adam                                                         & \{adam, \textbf{sgd}, sgd\_nesterov\_0.9\}                   & sgd                                                          \\
\texttt{optimizer\_online} & adam                                                         & \{adam, sgd, \textbf{sgd\_nesterov\_0.9}\}                   & sgd                                                          \\
\texttt{optimizer\_outer}  & adam                                                         & \{\textbf{adam}, adamw\}                                     & adam                                                         \\
\texttt{steps\_inner}      & $\{ 100, \mathbf{250}, 500, 1000 \}$                         & $\{ 100, 250, 500, \mathbf{1000} \}$                         & $\{ 5, \mathbf{10}, 50, 100 \}$                              \\
\texttt{steps\_nudged}     & $\{ \mathbf{100}, 250, 500, 1000 \}$                         & $\{ \mathbf{100}, 250, 500, 1000 \}$                         & -                                                            \\
\texttt{steps\_outer}      & 6400                                                         & 6400                                                         & 6400                                                         \\
$t_f$                      & $\{ 20, \mathbf{50}, 100 \}$                                 & $\{ 20, 50, \mathbf{100} \}$                                 & $\{ 20, 50, \mathbf{100} \}$                                 \\
$t_s$                      & $\{ 50, 100, \mathbf{250}, 500, 1000 \}$                     & $\{ 50, 100, 250, \mathbf{500}, 1000 \}$                     & $\{ 5, 10, 50, \mathbf{100} \}$                              \\ \bottomrule
\end{tabular}
\end{center}
\end{table}
\end{landscape}

\section{Additional details}

\subsection{Compute resources}
We used Linux workstations with 2 Nvidia RTX 3090 and 4 Nvidia RTX 3070 GPUs during development and conducted hyperparameter searches and larger experiments on up to 3 Linux servers with 8 Nvidia RTX 3090 GPUs with 24 GB memory each. Most of the experiments and corresponding hyperparameter scans presented take less than a few hours to complete on a single server. The more challenging recurrent spiking network and miniImageNet experiments require approximately 2-5 days to complete. During development we conducted many more hyperparameter scans over the course of several months.

\subsection{Software and libraries}
For the results produced in this paper we relied on free and open-source software.
We implemented our experiments in Python using PyTorch \citesupp[][BSD-style license]{paszke_pytorch_2019}, JAX \citesupp[][Apache License 2.0]{bradbury_jax_2018}, Ray \citesupp[][Apache License 2.0]{liaw_tune_2018} and NumPy \citesupp[][BSD-style license]{harris_array_2020}. For the visual few-shot classification dataset splits we used the Torchmeta library \citesupp[][MIT license]{deleu_torchmeta_2019} and for the generation of plots we used matplotlib \citesupp[][BSD-style license]{hunter_matplotlib_2007}.

\subsection{Datasets}
We conducted our experiments with the public domain datasets Boston housing \citesupp[][MIT License]{harrison_jr_hedonic_1978}, MNIST \citesupp[][GNU GPL v3.0]{lecun_mnist_1998}, Omniglot \citesupp[][]{lake_one_2011} (MIT license), miniImageNet \citesupp[][]{ravi_optimization_2016} (custom MIT/ImageNet license) and CIFAR-10 (MIT license) \citesupp{krizhevsky_learning_2009}.

%\begin{@fileswfalse}

\bibliographystylesupp{unsrtnat}
\bibliographysupp{references}
%\end{@fileswfalse}

\end{document}